\documentclass{article}
\usepackage[preprint]{colm2026_conference}

\usepackage[T1]{fontenc}
\usepackage{graphicx}
\usepackage{booktabs}
\usepackage{amsmath}
\usepackage{amssymb}
\usepackage{amsthm}
\usepackage{microtype}
\usepackage{hyperref}
\usepackage{url}
\usepackage{tikz}
\usetikzlibrary{matrix,positioning,calc}
\usepackage{pgfplots}
\usepackage{pgfplotstable}
\pgfplotsset{compat=1.18}
\usepackage{multirow}
\usepackage{xcolor}
\usepackage{colortbl}
\usepackage{array}
\usepackage{float}
\usepackage{placeins}
\usepackage{subcaption}

\newtheorem{proposition}{Proposition}

\newtheorem{definition}{Definition}

\pgfplotsset{
  hm colormap/.style={
    colormap={RdBu}{
      rgb255(0pt)  =(178,24,43);
      rgb255(25pt) =(239,138,98);
      rgb255(50pt) =(247,247,247);
      rgb255(75pt) =(103,169,207);
      rgb255(100pt)=(33,102,172);
    },
    colormap name=RdBu,
  }
}

\definecolor{darkblue}{rgb}{0, 0, 0.5}
\hypersetup{colorlinks=true, citecolor=darkblue, linkcolor=darkblue, urlcolor=darkblue}

\begin{document}


\pgfplotstableread{
	x  y  val
	1  5  -2.0
	2  5  +2.5
	3  5  +1.0
	4  5  -1.0
	5  5  +2.0
	6  5  -4.5
	1  4  +2.5
	2  4  +1.0
	3  4  +2.5
	4  4  +0.0
	5  4  -1.0
	6  4  -2.0
	1  3  +0.0
	2  3  +0.5
	3  3  +0.1
	4  3  +0.0
	5  3  +1.5
	6  3  -3.4
	1  2  +1.5
	2  2  -1.0
	3  2  +1.0
	4  2  +0.5
	5  2  -1.5
	6  2  -1.5
	1  1  +0.0
	2  1  +1.0
	3  1  +0.6
	4  1  +0.5
	5  1  +0.0
	6  1  -1.5
}\hmdata

\title{The Structural Attention Tax: How Retrieval Format\\Hijacks In-Context Learning Independent of Content}

\author{Yuqi Zhang, Di Zhang \\
Xi'an Jiaotong-Liverpool University \\
\texttt{yuqi.zhang23@student.xjtlu.edu.cn}
}

\maketitle

\begin{abstract}
	Retrieval-augmented generation (RAG) systems inject external knowledge
	to improve LLM outputs, yet the \emph{format} of injected
	content---distinct from its semantic relevance---can independently
	distort the model's attention distribution.
	We identify and formalise a phenomenon we term the
	\textbf{structural attention tax}: knowledge graph (KG) triples, due
	to their relational delimiters and repeated slot patterns, capture
	$2$--$3\times$ more attention per token than semantically equivalent
	natural-language text ($\hat{\sigma}(\text{KG}) \approx 0.70$ vs.\
	$\hat{\sigma}(\text{neutral}) \approx 0.25$), compressing
	demonstration attention by up to $42\%$---\emph{regardless of whether
		the triples are relevant or noise}.
	We develop a formal framework decomposing attention scores into
	semantic and structural components
	(Eq.~\ref{eq:structural_decomp}), derive a compression bound
	(Proposition~\ref{prop:compression}) connecting token-level format
	bias to demonstration attention loss, and show that the structural
	term governs \emph{how much} attention is diverted while the semantic
	term governs \emph{whether} this helps or hurts.
	This decoupling reveals two orthogonal axes for improving
	retrieval-augmented ICL: optimising retrieval quality (semantic axis)
	and reducing format-driven attention capture (structural axis).
	Empirically, across two model families (Mistral-7B, LLaMA-3-8B) and
	three QA benchmarks, we observe that source--task alignment dominates:
	task-matched BM25 retrieval achieves $58$--$62\%$ on HotpotQA vs.\
	ConceptNet's $25$--$27\%$, a $>$30\,pp gap that dwarfs all gating
	strategies ($\leq$2\,pp).
	We derive five structure-aware mitigation strategies from the
	framework, ranging from zero-cost prompt modifications to
	training-time regularisation; format flattening (S3) is validated by
	both accuracy and attention-level evidence from a verbalized-triple
	control, while structural dispersal (S1) yields mixed results that
	illuminate the challenges of format-level intervention.
\end{abstract}

\section{Introduction}

Retrieval-augmented generation (RAG)~\citep{jiang2023flare,sui2025fidelis}
has become a standard strategy for grounding LLM outputs in external
knowledge.
The dominant concern in RAG research is \emph{what} to retrieve:
selecting relevant passages~\citep{parry2024icl}, gating on
confidence~\citep{jiang2023flare}, or chaining knowledge graph facts
into reasoning traces~\citep{sui2025fidelis}.
Far less attention has been paid to a complementary question:
\emph{how does the format of retrieved content interact with the
	transformer's attention mechanism, independent of whether that content
	is semantically useful?}

We show that this question matters more than it might appear.
In transformer-based in-context learning
(ICL)~\citep{brown2020language,chen2025schema}, all prompt regions
compete for the same fixed attention budget.
When knowledge graph triples are injected into the prompt, their
distinctive structure---relational delimiters, repeated slot patterns,
high token-level regularity---creates a \emph{format-driven bias} in
attention allocation that operates independently of the triples'
semantic content.
We term this phenomenon the \textbf{structural attention tax}: a
systematic over-allocation of attention to structurally salient prompt
regions, with a corresponding under-allocation (compression) of
attention to demonstrations and other task-critical context.

Our central theoretical contribution is a decomposition of attention
scores into semantic and structural components
(Section~\ref{sec:theory}; Figure~\ref{fig:framework}), yielding a
formal characterisation of how format bias interacts with content
relevance:
\begin{itemize}
	\item The \textbf{structural component}
	      $\lambda \cdot \sigma(K)$ determines the \emph{magnitude} of
	      attention diversion---how many attention units are taxed away from
	      demonstrations.
	\item The \textbf{semantic component}
	      $\bar{s}_K^{\text{sem}}$ determines the \emph{sign} of the
	      performance effect---whether diverted attention carries useful signal
	      or noise.
\end{itemize}
This decoupling implies that optimising \emph{what} to retrieve
(semantic axis) and reducing \emph{format-driven capture}
(structural axis) are orthogonal improvement strategies, a perspective
that unifies several previously disconnected
observations~\citep{shi2023llm,wu2024faithful,liu2024rag}.

We develop this framework through four contributions:

\begin{enumerate}
	\item \textbf{The structural attention tax framework}
	      (Section~\ref{sec:theory}): a formal decomposition of attention
	      competition in augmented ICL, yielding four testable predictions and a
	      provable compression bound
	      (Proposition~\ref{prop:compression}).

	\item \textbf{Empirical validation of the format--content
		      decoupling} (Section~\ref{sec:results}): using a seven-condition
	      study across two model families (Mistral-7B, LLaMA-3-8B) and three QA
	      tasks, we show that KG triples absorb $2$--$3\times$ more attention
	      per token than neutral text, with noise and relevant triples
	      exhibiting nearly identical attention patterns---confirming that the
	      structural tax is format-driven, not content-driven.

	\item \textbf{A source-alignment dominance result}
	      (Section~\ref{sec:bm25}): task-matched BM25 retrieval outperforms
	      mismatched ConceptNet retrieval by $>$30\,pp on HotpotQA,
	      demonstrating that source selection along the semantic axis dwarfs
	      gating sophistication ($\leq$2\,pp). This result is currently limited
	      to one task and confounded by retrieval-unit differences
	      (Section~\ref{sec:limitations}).

	\item \textbf{Five structure-aware mitigation strategies}
	      (Section~\ref{sec:mitigation}): derived from the framework, targeting
	      the structural term $\lambda \cdot \sigma(K)$ through prompt
	      modification, logit suppression, and training-time regularisation.
	      Two strategies are empirically evaluated: S3 (format flattening) is
	      supported by both accuracy and attention-level evidence
	      (Appendix~\ref{app:c5b}); S1 (structural dispersal) yields mixed
	      results with model-dependent effects
	      (Appendix~\ref{app:s1_dispersal}).
	      The remaining three are mathematically grounded but untested.
\end{enumerate}

\textbf{Scope:}\;
We do not claim the structural attention tax makes KG augmentation
universally harmful; we argue its existence as an independent,
format-driven cost has been overlooked
(Section~\ref{sec:limitations}).

\section{Related Work}

\textbf{In-Context Learning.}\;
\citet{brown2020language} showed LLMs generalise from demonstrations
without gradient updates.
Demonstration format~\citep{min2022rethinking}, skill
matching~\citep{an2023skill}, and schema-structured
prompts~\citep{chen2025schema} affect performance.
Parametric fact recall degrades before ICL ability under
compression~\citep{jin2023cost}.
\citet{parry2024icl} frame ICL as applied information retrieval.

\textbf{KG Augmentation and RAG.}\;
Pipelines range from triple injection~\citep{li2023fewshot} to
graph-based reasoning~\citep{huang2023prodigy}.
FLARE~\citep{jiang2023flare} gates retrieval on confidence;
FiDeLiS~\citep{sui2025fidelis} chains KG facts into verifiable traces.
\citet{zheng2023edit} show KG triples serve as factual overrides;
\citet{wu2024faithful} quantify the ``tug-of-war'' between parametric
priors and retrieved evidence.
\citet{liu2024rag} find RAG gains are modest when chain-of-thought
already approaches correct conclusions;
\citet{shi2023llm} show LLMs are susceptible to distraction by
irrelevant context.
While these works identify cases where retrieval can hurt, none
isolate the \emph{format-driven} component of attention distortion
from the \emph{content-driven} component---the central contribution
of our framework.

\textbf{Multi-Hop Reasoning.}\;
Chain-of-thought~\citep{wei2022chain} and backward
chaining~\citep{kazemi2023lambada} improve multi-hop accuracy.
Confidence calibration~\citep{deng2023calibration} and the
memory--reasoning distinction~\citep{jin2025disentangling} motivate
our task contrast.

\textbf{Positioning.}\;
Our work introduces the \emph{structural attention tax} as a formal
concept, provides a decomposition framework that separates format
effects from content effects in retrieval-augmented ICL, and derives
mitigation strategies grounded in this decomposition.
This complements existing work on retrieval quality and
gating~\citep{jiang2023flare,wu2024faithful} by identifying an
orthogonal axis of improvement.

\section{The Structural Attention Tax Framework}
\label{sec:theory}

\begin{figure*}[t]
	\centering
	\begin{tikzpicture}
		\def\W{12}
		\def\H{0.75}
		
		\definecolor{cI}{RGB}{180,180,180}      
		\definecolor{cD}{RGB}{52,120,246}       
		\definecolor{cK}{RGB}{255,120,40}       
		\definecolor{cQ}{RGB}{46,204,113}       
		
		\def\ya{2.1}
		
		\node[anchor=south west, font=\small\itshape, text=black!60]
		at (0, \ya+\H+0.06)
		{Natural language:\;\upshape``A dog is a common pet. Cats are also animals.''};
		
		\fill[cI] (0,    \ya) rectangle (0.60,  \ya+\H);
		\fill[cD] (0.60, \ya) rectangle (4.80,  \ya+\H);
		\fill[cK!70] (4.80, \ya) rectangle (7.80,  \ya+\H);
		\fill[cQ] (7.80, \ya) rectangle (12.00, \ya+\H);
		
		\draw[black!25] (0, \ya) rectangle (12.00, \ya+\H);
		
		\node[font=\tiny, text=white] at (0.30, \ya+\H/2) {I};
		\node[font=\scriptsize, text=white] at (2.70, \ya+\H/2)
		{Demonstrations \textbf{35\%}};
		\node[font=\scriptsize, text=black] at (6.30, \ya+\H/2)
		{Knowledge \textbf{25\%}};
		\node[font=\scriptsize, text=white] at (9.90, \ya+\H/2)
		{Query \textbf{35\%}};
		
		\node[font=\scriptsize, text=black!50] at (6.0, \ya-0.3)
		{$\updownarrow$ same knowledge content, different format};
		
		\def\yb{0}
		
		\node[anchor=south west, font=\small\itshape, text=black!60]
		at (0, \yb+\H+0.06)
		{KG triples:\;\upshape\texttt{dog|IsA|pet;\; cat|IsA|animal}};
		
		\fill[cI] (0,    \yb) rectangle (0.48,  \yb+\H);
		\fill[cD!80] (0.48, \yb) rectangle (2.88, \yb+\H);
		\fill[cK] (2.88, \yb) rectangle (8.64, \yb+\H);
		\fill[cQ] (8.64, \yb) rectangle (12.00, \yb+\H);
		
		\draw[black!25] (0, \yb) rectangle (12.00, \yb+\H);
		
		\node[font=\tiny, text=white] at (0.24, \yb+\H/2) {I};
		\node[font=\scriptsize, text=white] at (1.68, \yb+\H/2)
		{D \textbf{20\%}};
		\node[font=\small, text=black] at (5.76, \yb+\H/2)
		{\textbf{Knowledge 48\%}};
		\node[font=\scriptsize, text=white] at (10.32, \yb+\H/2)
		{Q \textbf{28\%}};
		
		\draw[dashed, black!40, line width=0.9pt]
		(4.80, \ya) -- (4.80, \yb+\H);
		
		\draw[decorate,
		decoration={brace, amplitude=5pt, mirror},
		thick, cK]
		(2.88, -0.14) -- (8.64, -0.14)
		node[midway, below=7pt, font=\small\bfseries, text=cK]
		{Structural Attention Tax};
		
	\end{tikzpicture}
	
	\caption{The structural attention tax.
		Each bar shows how the first answer token's last-layer attention
		(summing to 100\%) is distributed across four prompt regions:
		instruction~(I), demonstrations~(D), knowledge~(K), and query~(Q).
		The knowledge content is identical in both rows; only the \emph{format}
		differs.
		Presenting knowledge as KG triples (bottom) nearly doubles the
		knowledge-region attention share ($25\% \to 48\%$) and compresses
		demonstration attention from $35\%$ to $20\%$, regardless of whether
		the triples are semantically relevant.
		The dashed line marks the natural-language knowledge boundary;
		attention to its left has been ``taxed away'' from demonstrations.
		(Illustrative values; see Section~\ref{sec:results} for measured data.)}

	\label{fig:framework}
\end{figure*}
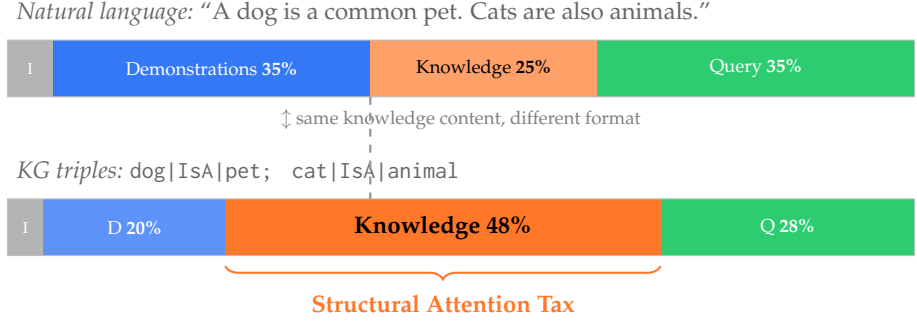

We develop a formal account of how the \emph{format} of injected
knowledge interacts with the transformer's attention mechanism,
generating four testable predictions (see Figure~\ref{fig:framework}
for an overview).
The decompositions serve as heuristic frameworks that organise
otherwise disconnected observations into a coherent theory.

\subsection{Setup and Notation}

Let $q$ denote a query with gold answer $y^*$.
The prompt is $x = [I;\, D;\, K;\, q]$ with instruction $I$,
demonstrations $D$, optional knowledge $K$, and question $q$.
Define $c_0(q) \triangleq p(y^* \mid x_{\varnothing})$ (no knowledge)
and $c_K(q) \triangleq p(y^* \mid x_K)$ (with knowledge).

\subsection{Attention Score Decomposition}
\label{sec:theory_attention}

For query token $i$ at layer $l$, the attention mass on $K$ is
$A_K^{(l)}(i) = \sum_{j \in K} \exp(s_{ij}^{(l)}) / \sum_k \exp(s_{ik}^{(l)})$.
Since attention is normalised, $A_D + A_K + A_I + A_Q = 1$.
We decompose the attention score into semantic and structural
components:
\begin{equation}
	s_{ij}^{(l)} =
	\underbrace{s_{ij}^{(l),\text{sem}}}_{\text{content-driven}}
	+ \underbrace{b_j^{(l)}}_{\text{format bias}}.
	\label{eq:score_decomp}
\end{equation}
The effective attention allocated to $K$ decomposes as:
\begin{equation}
	A_K^{(l,h)}(i) =
	\underbrace{A_K^{(l,h),\text{sem}}(i)}_{\text{semantic relevance}}
	+ \underbrace{\lambda^{(l,h)} \cdot \sigma(K)}_{\text{structural attention tax}},
	\label{eq:structural_decomp}
\end{equation}
where $\sigma(K) \in [0,1]$ quantifies \textbf{structural intensity}
(triple density, delimiter frequency, slot repetitiveness) and
$\lambda^{(l,h)}$ is a model-intrinsic bias coefficient.
The term $\lambda \cdot \sigma(K)$ is the \emph{structural attention
	tax}: attention captured by format alone, independent of whether the
content is relevant, irrelevant, or noise.

\begin{definition}[Structural capture potential]
	\label{def:sigma}
	For region $\mathcal{R}$ with $m$ tokens:
	$\sigma(\mathcal{R}) = \gamma \cdot \frac{1}{m}
		\sum_{j \in \mathcal{R}}
		\mathbb{I}[\text{token}_j \in \mathcal{P}_{\text{struct}}]
		+ \beta_{\text{rep}} \cdot \text{rep}(\mathcal{R})$,
	where $\mathcal{P}_{\text{struct}}$ is the structured-pattern token set
	(relation keywords, delimiters, slot markers) and
	$\text{rep}(\mathcal{R})$ quantifies repetitiveness.
\end{definition}

\textbf{The key insight} is that the structural and semantic
components play fundamentally different roles:
\emph{$\lambda \cdot \sigma(K)$ determines how much attention is
taxed; $\bar{s}_K^{\text{sem}}$ determines whether this tax helps or
hurts.}
This decoupling generates two orthogonal improvement axes: reducing
format-driven capture (targeting $\sigma(K)$ or $\lambda$) and
improving retrieval quality (targeting $\bar{s}_K^{\text{sem}}$).

\subsection{Demonstration Compression Bound}
\label{sec:compression}

The zero-sum constraint implies that the structural tax compresses
demonstration attention:
\begin{equation}
	A_D^{(l),\text{eff}} = A_D^{(l),\text{sem}}
	- \eta \cdot \lambda \cdot \sigma(K) \cdot
	\frac{A_D^{(l),\text{sem}}}
	{\sum_{R \neq K} A_R^{(l),\text{sem}}},
	\label{eq:demo_compression_structural}
\end{equation}
where $\eta \in [0.5, 1.0]$ is a competition coefficient.

\begin{proposition}[Demonstration compression bound]
	\label{prop:compression}
	If $K$ has $m$ tokens with mean logit $\bar{s}_K$ and $D$ has mean
	logit $\bar{s}_D$, then:
	\begin{equation}
		\frac{A_D^{(K)}}{A_D^{(0)}} \geq
		\frac{1}{1 + \frac{m}{T_0} \cdot
			\exp(\bar{s}_K - \bar{s}_D)}.
		\label{eq:compression_bound}
	\end{equation}
\end{proposition}
\vspace{-2pt}
Incorporating the structural decomposition, $\bar{s}_K =
	\bar{s}_K^{\text{sem}} + \lambda \cdot \sigma(K)$, so:
\begin{equation}
	\frac{A_D^{(K)}}{A_D^{(0)}} \geq
	\frac{1}{1 + \frac{m}{T_0} \cdot
		\exp\!\bigl(\bar{s}_K^{\text{sem}} + \lambda \cdot \sigma(K) - \bar{s}_D\bigr)}.
	\label{eq:compression_structural}
\end{equation}
The structural term $\lambda \cdot \sigma(K)$ appears inside the
exponential, meaning even modest format bias is \emph{amplified
	exponentially} in its effect on compression.
When $\lambda \cdot \sigma(K) \gg |\bar{s}_K^{\text{sem}} - \bar{s}_D|$,
noise and relevant triples compress demonstrations nearly identically
---a signature prediction of the structural attention tax.

\subsection{Source--Task Alignment}
\label{sec:theory_info}

The influence of $K$ on the model's output decomposes into a useful
signal component and a distraction component
(Appendix~\ref{app:mi_decomposition}).
\textbf{Prediction 1 (Source dominance):} When $K$ is misaligned
(e.g., ConceptNet for Wikipedia-based questions), distraction
dominates.
The structural attention tax amplifies this: misaligned triples are
not merely uninformative but \emph{actively costly} because they
impose a format-driven attention tax on top of semantic distraction.

\subsection{Confidence-Dependent Interference}
\label{sec:theory_confidence}

The KL divergence $D_{\mathrm{KL}}(p(\cdot \mid x_K) \| p(\cdot \mid x_{\varnothing}))$ quantifies the representational shift from knowledge injection.
In the \emph{high-confidence regime} ($c_0 \to 1$), any shift leaks
mass away from $y^*$:
\begin{equation}
	c_0(q) \approx 1 \implies
	c_K(q) \leq c_0(q) -
	\underbrace{\textstyle\sum_{y \neq y^*}
		[p(y \mid x_K) - p(y \mid x_{\varnothing})]^+}_{\triangleq\,
		\ell(q, K)\,\geq\, 0}.
	\label{eq:highconf}
\end{equation}
\textbf{Prediction 2 (Confidence modulation):} The expected accuracy change is:
\begin{equation}
	\mathbb{E}[\Delta] \approx
	\underbrace{P(c_0 \ll 1) \cdot \mathbb{E}[\Delta \mid c_0 \ll 1]}_{>0}
	+
	\underbrace{P(c_0 \approx 1) \cdot \mathbb{E}[\Delta \mid c_0 \approx 1]}_{<0}.
	\label{eq:expected_delta}
\end{equation}
The structural tax exacerbates this: even when $K$ contains useful
signal, format-driven attention capture reduces the model's ability to
attend to demonstrations that provide task-critical calibration.

\subsection{Testable Predictions}
\label{sec:predictions}

The framework generates four predictions:
\textbf{P1}~(Source dominance): task-matched retrieval outperforms
mismatched;
\textbf{P2}~(Confidence modulation): KG hurts high-confidence tasks,
helps low-confidence tasks;
\textbf{P3}~(Format-invariant capture): noise and relevant triples
absorb similar attention because $\lambda \cdot \sigma(K)$ dominates;
\textbf{P4}~(Compression--performance decoupling): KG-broken and
KG-fixed samples show similar attention but divergent accuracy.

\section{Methodology}
\label{sec:methodology}

\subsection{Experimental Design}
\label{sec:expdesign}

The prompt takes the form
$x = [\text{Instr.};\; \text{Demos};\; T(q);\; q]$ with $k{=}3$
demonstrations held constant.
We design seven conditions to isolate the structural attention tax
from semantic effects:

\textbf{C1}~(Standard ICL, no external context);
\textbf{C2}~(Relevant-KG: top-3 ConceptNet triples by cosine
similarity, \texttt{all-MiniLM-L6-v2});
\textbf{C3}~(Noise-KG: unrelated triples);
\textbf{C4}~(Scalar-Gated: inject when first-token log-prob $< -0.3$);
\textbf{C5}~(Neutral-Text: length-matched neutral sentences);
\textbf{C5b}~(Verbalized triples; Appendix~\ref{app:c5b});
\textbf{C6}~(Multi-Feature Gate: dual inference, diagnostic only);
\textbf{C7}~(BM25 Wikipedia passage, HotpotQA only).

The C2/C3/C5 contrast is the key design for isolating the structural tax:
C2 and C3 share high $\sigma(K)$ but differ in semantic relevance;
C5 has low $\sigma(K)$ but matched token count.
If the structural tax exists, C2 and C3 should show similar attention
capture despite opposite semantic relevance, and both should exceed C5.

\emph{Caveat:} C2 optimises for surface similarity rather than
answer-discriminative relevance; stronger KG retrieval methods might
yield different results.

\subsection{Datasets, Models, and Metrics}
\label{sec:datasets}

We evaluate on \textbf{CommonsenseQA}~\citep{talmor2019commonsenseqa}
(five-way MC),
\textbf{HotpotQA}~\citep{yang2018hotpotqa} (multi-hop), and
\textbf{TriviaQA}~\citep{joshi2017triviaqa} (open-domain factual).
Base set: $n{=}200$; expanded to $n{=}1{,}000$ for McNemar tests.
Models: \textbf{Mistral-7B-Instruct-v0.1}~\citep{jiang2023mistral} and
\textbf{LLaMA-3-8B-Instruct}~\citep{dubey2024llama3}, both with 4-bit
NF4 quantisation and greedy decoding.
Quantisation introduces perturbations: FP16 shows $+10$\,pp on
HotpotQA C1 (Appendix~\ref{app:fp16}), so fine-grained effects
(1--3\,pp) should be interpreted cautiously.
Metrics: exact-match accuracy, first-token log-probability as
confidence proxy, McNemar's test with Bonferroni correction.

\emph{Confidence proxy caveat:} First-token log-probability conflates
knowledge confidence with format preferences and tokeniser biases;
cross-task comparability is limited.

\section{Results}
\label{sec:results}

We organise results around the four predictions of the structural
attention tax framework, leading with the core structural findings
(Predictions~3--4), then examining performance effects
(Prediction~2), and concluding with the source-alignment result
(Prediction~1).

\subsection{The Structural Tax in Action (Predictions 3--4)}
\label{sec:structural_results}

\textbf{Format-invariant attention capture (Prediction 3).}\;
Last-layer attention shows KG triples absorbing 7--10\% (Mistral) or
3--6\% (LLaMA) of attention mass \emph{regardless of relevance}, with
demonstration compression up to 42\%
(Figure~\ref{fig:attn_capture}).
Critically, C3 (noise triples) shows attention absorption comparable
to C2 (relevant triples)---7.7\% vs.\ 10.3\% for Mistral on
CSQA---confirming that the structural tax operates independently of
content (Figure~\ref{fig:c2c3_compare}).
In contrast, C5 (neutral text) absorbs only $\sim$3\%, yielding the
signature $\sim$3$\times$ ratio:
$\hat{\sigma}(\text{KG}) \approx 0.70$ vs.\
$\hat{\sigma}(\text{neutral}) \approx 0.25$.

\begin{figure}[t]
	\centering
	\begin{subcaptiongroup}
		\begin{minipage}[b]{0.48\columnwidth}
			\centering
			\includegraphics[width=\linewidth]{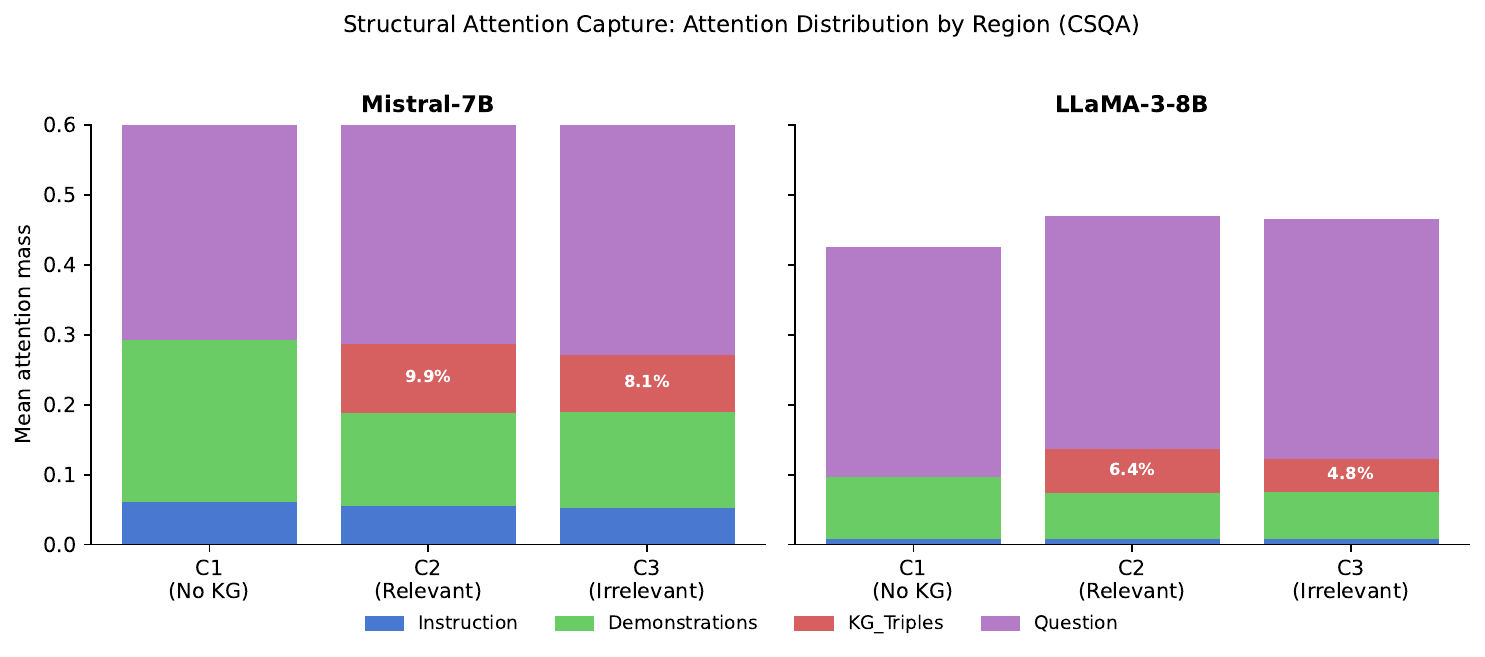}
			\subcaption{Attention by region under C1/C2/C3.}
			\label{fig:attn_capture}
		\end{minipage}\hfill
		\begin{minipage}[b]{0.48\columnwidth}
			\centering
			\includegraphics[width=\linewidth]{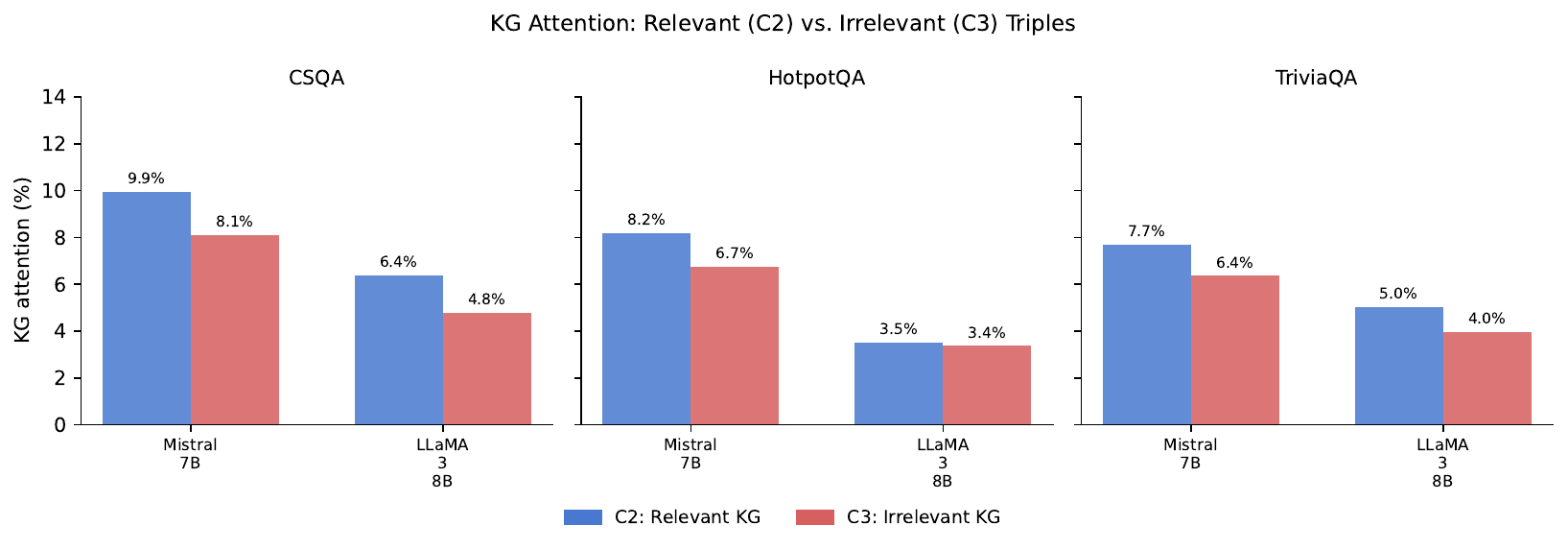}
			\subcaption{C2 vs.\ C3 KG-region attention.}
			\label{fig:c2c3_compare}
		\end{minipage}
	\end{subcaptiongroup}
	\caption{Structural attention tax evidence.
		(a)~KG triples capture 7--10\% of attention regardless of relevance,
		compressing demonstration attention by up to 42\%.
		(b)~Similar capture rates for relevant (C2) and irrelevant (C3) triples
		confirm format-driven allocation (Prediction~3).}
	\label{fig:structural_evidence}
\end{figure}

\textbf{Quantifying the tax.}\;
KG tokens receive $\approx 0.34\%$/token (Mistral) vs.\
$0.13\%$/token for demonstrations ($2.6\times$), with estimated
structural bias
$\lambda \approx 0.07$--$0.10$ (Mistral) and $0.03$--$0.06$ (LLaMA).
The exponential amplification in
Eq.~\ref{eq:compression_structural} explains why even modest $\lambda
	\cdot \sigma(K) \approx 0.05$--$0.07$ produces substantial compression.

\textbf{Compression--performance decoupling (Prediction 4).}\;
KG-broken and KG-fixed samples show structurally similar attention
redistribution patterns (Table~\ref{tab:attention}), confirming that
compression magnitude alone does not predict performance direction.
This validates the core claim of the framework: the structural tax
determines \emph{how much}; the semantic content determines
\emph{whether}.

\textbf{Neutral-text control (C5).}\;
C5 decomposes the C2 effect into crowding (C1$-$C5) and semantic
(C5$-$C2) components (most estimates 1--3\,pp, within CI at $n{=}200$).
On CSQA, neutral text has negligible effect while KG triples degrade
performance---consistent with content-driven interference
\emph{amplified by} the structural tax.
On HotpotQA, KG triples overcome crowding for a net gain.

These are correlational last-layer observations; multi-layer analysis
(Appendix~\ref{app:multilayer}) confirms KG attention elevation
persists across layers; causal validation (attention masking,
activation patching) is important future work.

\subsection{Confidence-Dependent Effects (Prediction 2)}

Table~\ref{tab:accuracy} reports accuracy across conditions
(95\% CI $\approx \pm 6.5$\,pp at $n{=}200$).
On \textbf{CSQA} (high parametric confidence), C2 shows a directional
decrease ($-2.0$ and $-1.0$\,pp), consistent with
Eq.~\ref{eq:highconf}: triples containing plausible competing entities
redistribute mass away from $y^*$.
On \textbf{HotpotQA} (low confidence), C2 improves accuracy
($+2.5/+2.0$\,pp), consistent with Eq.~\ref{eq:expected_delta}.
On \textbf{TriviaQA}, results are mixed; LLaMA shows $-4.5$\,pp.
All differences except one fall within the CI.

McNemar's test at $n{=}1{,}000$ with Bonferroni correction
($\alpha_{\text{adj}} = 0.0083$; Table~\ref{tab:mcnemar}): only
\textbf{LLaMA-3-8B HotpotQA} survives
($p_{\text{Bonf}} = 0.001$; OR\,$= 3.36$).
A sign test across all six outcomes yields $p = 0.69$
(non-significant), so the confidence-dependent pattern remains a
directional trend requiring replication.

\begin{table}[t]
	\caption{Accuracy (\%, $n=200$; 95\% CI $\approx \pm 6.5$\,pp).
		Best per pair in \textbf{bold}.}
	\label{tab:accuracy}
	\centering
	\small
	\setlength{\tabcolsep}{3.5pt}
	\begin{tabular}{llcccccc}
		\toprule
		Model & Dataset  & C1            & C2            & C3            & C4   & C5            & C6            \\
		\midrule
		\multirow{3}{*}{Mistral-7B}
		      & CSQA     & 70.0          & 68.0          & \textbf{72.5} & 70.0 & 71.5          & 70.0          \\
		      & HotpotQA & 22.0          & \textbf{24.5} & 23.0          & 22.5 & 21.0          & 23.0          \\
		      & TriviaQA & 71.5          & 72.5          & \textbf{74.0} & 71.6 & 72.5          & 72.1          \\
		\midrule
		\multirow{3}{*}{LLaMA-3-8B}
		      & CSQA     & 73.0          & 72.0          & 73.0          & 73.0 & \textbf{73.5} & \textbf{73.5} \\
		      & HotpotQA & 24.5          & \textbf{26.5} & 23.5          & 26.0 & 23.0          & 24.5          \\
		      & TriviaQA & \textbf{74.5} & 70.0          & 72.5          & 71.1 & 73.1          & 73.1          \\
		\bottomrule
	\end{tabular}
\end{table}

\begin{table}[t]
	\caption{McNemar's test (C2 vs C1, $n{=}1{,}000$).
	$\dagger$ = $p_{\text{Bonf}} < 0.05$.
	$\ddagger$ = $p_{\text{raw}} < 0.05$ only.}
	\label{tab:mcnemar}
	\centering
	\small
	\setlength{\tabcolsep}{2.6pt}
	\begin{tabular}{llcccccc}
		\toprule
		Model & Dataset          & $n$               & $n_{\text{fix}}$ & $n_{\text{brk}}$
		      & $p_{\text{raw}}$ & $p_{\text{Bonf}}$ & OR                                                                                         \\
		\midrule
		\multirow{3}{*}{Mistral-7B}
		      & CSQA             & 1000              & 40               & 61               & .046$^{\ddagger}$ & .276                      & 0.66 \\
		      & HotpotQA         & 1000              & 28               & 21               & .392              & 1.00                      & 1.33 \\
		      & TriviaQA         & 1000              & 26               & 39               & .136              & .816                      & 0.67 \\
		\multirow{3}{*}{LLaMA-3-8B}
		      & CSQA             & 1000              & 38               & 61               & .027$^{\ddagger}$ & .159                      & 0.62 \\
		      & HotpotQA         & 1000              & 37               & 11               & $\mathbf{.000}$   & $\mathbf{.001}^{\dagger}$ & 3.36 \\
		      & TriviaQA         & 1000              & 17               & 35               & .018$^{\ddagger}$ & .105                      & 0.49 \\
		\bottomrule
	\end{tabular}
\end{table}

\subsection{Source--Task Alignment Dominates (Prediction 1)}
\label{sec:bm25}

Having established that the structural tax is format-driven and
measurable, we now examine the complementary semantic axis.
Table~\ref{tab:bm25} reports C7 on HotpotQA.
Mistral achieves 61.5\% under C7 vs.\ 24.5\% under C2
($+37.0$\,pp); LLaMA achieves 58.0\% vs.\ 26.5\% ($+31.5$\,pp).
The $>$30\,pp gap is an order of magnitude larger than any gating
strategy ($\leq$2\,pp), strongly supporting Prediction~1.

In the language of our framework, Wikipedia retrieval for HotpotQA
maximises $\bar{s}_K^{\text{sem}}$ (task-aligned content) while
presenting information in coherent prose with lower $\sigma(K)$
than triple format, reducing both semantic distraction and the
structural attention tax simultaneously.

\emph{Caveats:} C2 and C7 differ in retrieval unit, token budget, and
text coherence; C7 is limited to HotpotQA; a stronger KG pipeline
might narrow the gap.
Gating strategies (C4, C6) yield $\leq$2\,pp improvements, negligible
compared to source selection (Appendix~\ref{app:oracle}).

\begin{table}[t]
	\caption{C7 on HotpotQA ($n{=}200$). The $>$30\,pp gap supports
		Prediction~1 (source dominance). C1 values differ from
		Table~\ref{tab:accuracy} (separate split).}
	\label{tab:bm25}
	\centering
	\small
	\begin{tabular}{lccc}
		\toprule
		Model      & C1   & C2 (CN) & C7 (BM25)     \\
		\midrule
		Mistral-7B & 28.0 & 24.5    & \textbf{61.5} \\
		LLaMA-3-8B & 21.5 & 26.5    & \textbf{58.0} \\
		\bottomrule
	\end{tabular}
\end{table}

\section{Structure-Aware Mitigation Strategies}
\label{sec:mitigation}

The structural attention tax framework identifies
$\lambda \cdot \sigma(K)$ as a format-driven attention cost
independent of content quality.
This suggests a principled design space for mitigation, targeting
different terms in Eq.~\ref{eq:structural_decomp}.
S3 is supported by both accuracy and attention-level evidence from the
C5b condition (Appendix~\ref{app:c5b}); S1 is empirically tested with
mixed results (Appendix~\ref{app:s1_dispersal}); the remaining three
strategies are \textbf{untested} framework-derived hypotheses.

\textbf{S1: Structural Dispersal} (targets $\sigma(K)\!\downarrow$).\;
Interleave triples with natural-language prose, reducing contiguous
structural density: $\sigma_{\text{disp}} = \sigma_{\text{orig}} /
	\rho$ where $\rho > 1$ is the dispersal factor.
Predicted compression reduction: $\Delta_D' \approx \Delta_D / \rho$.
A pilot experiment (Appendix~\ref{app:s1_dispersal}) reveals that
dispersal effects are model-dependent: LLaMA-3-8B shows the predicted
20--39\% reduction in KG attention, but Mistral-7B shows a
paradoxical \emph{increase}, and accuracy degrades in both cases
($-0.5$ to $-6.5$\,pp), suggesting that bridging phrases can
introduce new structural anchors that offset the intended dispersal.

\textbf{S2: Attention Logit Suppression} (targets $\bar{s}_K\!\downarrow$).\;
Before softmax, subtract $c > 0$ from KG-region logits:
$s_{ij}^{(l)} \leftarrow s_{ij}^{(l)} - c \cdot \mathbb{I}[j \in K]$.
This directly counteracts the structural tax:
\begin{equation}
	\frac{A_D^{(K,c)}}{A_D^{(0)}} \geq
	\frac{1}{1 + \frac{m}{T_0} \cdot
		\exp(\bar{s}_K - c - \bar{s}_D)}.
	\label{eq:suppressed_bound}
\end{equation}
With $\bar{s}_K - \bar{s}_D \approx 0.96$, $c \in [0.5, 1.5]$ should
substantially reduce compression while preserving semantic signal.

\textbf{S3: Format Flattening} (targets $\sigma(K)\!\downarrow$).\;
Convert triples to natural sentences, reducing $\sigma(K)$
by removing delimiter patterns and slot structure.
The C5b condition (Appendix~\ref{app:c5b}) provides direct support:
verbalized triples not only maintain accuracy on 4 of 6
model--task pairs, but also reduce KG-region attention by
17--29\% on LLaMA-3-8B (ratio C5b/C2 $\approx 0.71$--$0.83$),
confirming that format flattening lowers the structural tax while
preserving semantic content.
Extension: interrogative form (e.g., ``\textit{Have you considered
	that a rug is often on a floor?}'') further aligns with demonstration
style.

\textbf{S4: Confidence-Modulated Injection} (targets effective attention).\;
Define trust coefficient $\mu(q) \in [0,1]$:
$A_K^{\text{eff}} = \mu(q) \cdot A_K^{\text{sem}} + (1{-}\mu(q))
	\cdot \lambda \cdot \sigma(K)$.
A meta-instruction (``\textit{These facts are for reference only}'')
encourages high $\mu(q)$ for confident queries.

\textbf{S5: Structural Adversarial Regularisation} (targets $\lambda\!\downarrow$).\;
During fine-tuning, penalise attention to noise triples:
\begin{equation}
	\mathcal{L}_{\text{struct}} = \alpha_{\text{reg}} \sum_l \sum_i
	\frac{A_K^{(l)}(i; K_{\text{noise}})}{A_D^{(l)}(i) + \epsilon}.
	\label{eq:adversarial_reg}
\end{equation}
This directly reduces $\lambda$, the model-intrinsic format bias---the
most expensive but most durable approach.

\begin{table}[t]
	\caption{Mitigation strategies derived from the structural attention
		tax framework.  S3 is supported by C5b accuracy and attention evidence
		(Appendix~\ref{app:c5b}); S1 shows mixed results
		(Appendix~\ref{app:s1_dispersal}); S2/S4/S5 are untested.}
	\label{tab:strategies}
	\centering
	\small
	\setlength{\tabcolsep}{3pt}
	\begin{tabular}{llcccc}
		\toprule
		Strategy        & Target                  & Cost           & Access    & Eq.                       & Status    \\
		\midrule
		S1 Dispersal    & $\sigma(K)\!\downarrow$ & None           & Prompt    & --                        & Mixed     \\
		S2 Logit suppr. & $\bar{s}_K\!\downarrow$ & Min.           & Inference & \ref{eq:suppressed_bound} & Untested  \\
		S3 Flattening   & $\sigma(K)\!\downarrow$ & None           & Prompt    & --                        & Supported \\
		S4 Conf.\ mod.  & $\mu(q)\!\uparrow$      & $1$--$2\times$ & Prompt    & --                        & Untested  \\
		S5 Adv.\ reg.   & $\lambda\!\downarrow$   & Training       & Weights   & \ref{eq:adversarial_reg}  & Untested  \\
		\bottomrule
	\end{tabular}
\end{table}

Table~\ref{tab:strategies} summarises all five strategies.
The strategies form a cost--effectiveness ladder: S1/S3 require only
prompt modification, S2 requires logit access, S4 requires dual
inference, S5 requires fine-tuning.
This structured design space is a direct consequence of the
decomposition in Eq.~\ref{eq:structural_decomp}: each strategy
targets a specific term, enabling principled selection based on
deployment constraints.
Notably, the contrasting outcomes of S1 (mixed) and S3 (supported)
highlight that \emph{how} $\sigma(K)$ is reduced matters: eliminating
structural patterns entirely (S3) is more reliable than diluting them
with additional tokens that may themselves become attention anchors (S1).

\section{Discussion}

\textbf{The structural attention tax as a unifying concept.}\;
Prior work has noted that retrieval can hurt~\citep{shi2023llm},
that parametric and retrieved knowledge compete~\citep{wu2024faithful},
and that RAG gains diminish when models are already
confident~\citep{liu2024rag}.
Our framework unifies these observations by identifying a single
mechanism---format-driven attention capture---that operates
orthogonally to content quality.
The $\sim$3$\times$ ratio between KG-format and neutral-text attention
capture ($\hat{\sigma}(\text{KG}) / \hat{\sigma}(\text{neutral})$)
confirms that triple format elevates attention independently of
content, and the exponential amplification in
Eq.~\ref{eq:compression_structural} explains why this modest bias
produces substantial downstream effects.

\textbf{Two orthogonal axes.}\;
The $>$30\,pp BM25 gap demonstrates the semantic axis; the
$\sim$3$\times$ structural ratio identifies an untapped structural
axis.
Current RAG research focuses almost exclusively on the former; our
framework motivates systematic attention to the latter.

\textbf{Practical guidelines.}\;
(1)~Match knowledge source to task---our strongest finding, though
currently limited to one task.
(2)~If parametric confidence is high, avoid KG injection from
mismatched sources.
(3)~Apply format flattening (zero-cost, empirically supported) rather
than structural dispersal (which may introduce new attention anchors).
(4)~Use logit suppression when attention access is available.
(5)~For deployment, consider structural adversarial regularisation.

\textbf{Broader implications.}\;
The structural attention tax may extend to any prompt component with
distinctive formatting (SQL, JSON, code blocks), suggesting that
format normalisation should be a standard preprocessing step in RAG
pipelines.

\section{Limitations}
\label{sec:limitations}

\textbf{Statistical power:} Only one of six Bonferroni-corrected
comparisons is significant ($p = 0.69$ sign test); the
confidence-dependent pattern is a directional trend.
\textbf{Retrieval:} Cosine-similarity ConceptNet retrieval is
relatively weak; BM25 evaluated only on HotpotQA.
\textbf{Scale:} Two 7B/8B models, 4-bit NF4 quantisation
(FP16 shows $+10$\,pp on HotpotQA C1), greedy decoding.
\textbf{Attention:} Last-layer, correlational, no causal intervention.
\textbf{Theory:} Eq.~\ref{eq:mi_decomp} is a heuristic;
Eq.~\ref{eq:structural_decomp} assumes additive separation.
\textbf{Mitigation:} S3 is supported by both accuracy and attention
evidence; S1 yields mixed results with model-dependent effects;
the remaining three strategies (S2, S4, S5) are untested.
\textbf{Generality:} Demonstrated only for KG triple format.
See Appendix~\ref{app:extended_limitations} for extended discussion.

\section{Conclusion}

We have introduced the \textbf{structural attention tax}: a
format-driven mechanism by which structured prompt regions (such as
knowledge graph triples) capture disproportionate attention
independent of their semantic content, compressing demonstration
attention by up to 42\%.
Our formal framework decomposes attention competition into semantic and
structural components, revealing that these two axes govern different
aspects of retrieval-augmented ICL: the semantic term determines
\emph{whether} augmentation helps; the structural term determines
\emph{how much} attention is taxed.

This decoupling yields three actionable insights.
First, source--task alignment along the semantic axis dominates:
task-matched BM25 retrieval achieves $58$--$62\%$ on HotpotQA vs.\
ConceptNet's $25$--$27\%$ ($>$30\,pp gap).
Second, the structural tax is real and measurable:
$\hat{\sigma}(\text{KG}) \approx 3\times \hat{\sigma}(\text{neutral})$,
with noise and relevant triples showing comparable attention capture.
Third, the framework generates a principled design space of five
mitigation strategies; empirical evaluation of two prompt-level
strategies reveals that format flattening (S3) effectively reduces the
structural tax (17--29\% KG attention reduction on LLaMA-3-8B) while
preserving accuracy, whereas structural dispersal (S1) produces
model-dependent effects with accuracy degradation, highlighting
that the design of format-level interventions requires care to avoid
introducing new structural anchors.

\textbf{Future work:} (i)~causal interventions (attention masking,
activation patching) to validate the structural tax mechanism;
(ii)~extend C7 to other tasks;
(iii)~empirically evaluate the remaining three mitigation strategies
(S2, S4, S5);
(iv)~investigate the structural tax for other formatted prompt
components (SQL, JSON, code blocks).

\bibliographystyle{plainnat}
\bibliography{references}

\clearpage
\appendix
\section*{Appendix}

\section{Experimental Setup and Validation}
\label{app:setup_validation}

\subsection{Extended Condition Descriptions}
\label{app:conditions}

\textbf{C2:} Triples ranked by cosine similarity
(\texttt{all-MiniLM-L6-v2}), rendered as natural language.
\textbf{C3:} Shares relational format with C2
($\sigma(\text{C3}) \approx \sigma(\text{C2})$), enabling Prediction~3
testing.
\textbf{C4:} $\tau = -0.3$ via grid search on 50 held-out samples.
\textbf{C5:} Low $\sigma(\text{C5}) \ll \sigma(\text{C2})$; does not
control for syntactic form.
\textbf{C6:} $\tau_a$, $\delta$ via 5-fold CV.
\textbf{C7:} HotpotQA only.
Entity extraction: spaCy NER + noun-chunk detector.

\subsection{FP16 vs.\ INT4 Quantisation Comparison}
\label{app:fp16}

Table~\ref{tab:fp16} compares FP16 and INT4 accuracy on a subset of HotpotQA.
The $+10$\,pp gap under C1 indicates that quantisation substantially affects
baseline performance, warranting caution when interpreting small effect sizes
in the main experiments.

\begin{table}[htbp]
	\caption{FP16 vs.\ INT4 (Mistral-7B, HotpotQA, $n{=}50$).}
	\label{tab:fp16}
	\centering
	\small
	\begin{tabular}{lccc}
		\toprule
		Condition & INT4 (\%) & FP16 (\%) & $\Delta$ \\
		\midrule
		C1        & 22.0      & 32.0      & $+10.0$  \\
		C2        & 24.5      & 28.0      & $+3.5$   \\
		\bottomrule
	\end{tabular}
\end{table}

\subsection{Alias-Aware Evaluation}
\label{app:alias}

Alias matching yields uniform improvements ($\leq$3\,pp) not altering
conclusions.

\subsection{SARP Threshold Stability}
\label{app:sarp}

Bootstrap stability (200 resamples) confirms $\tau=-0.3$ is almost
never optimal ($<4\%$); modal: $\tau=-0.05$ (62--100\%).

\section{Supplementary Results}
\label{app:supp_results}

\subsection{Additional Main Results}
\label{app:additional_results}

This section reports supplementary metrics that complement the main accuracy
results.
Table~\ref{tab:logprob} reports mean answer log-probabilities, serving as
a confidence proxy across conditions.
Table~\ref{tab:errorflow} traces per-sample error transitions from C1 to C2.
Figure~\ref{fig:heatmap} visualises accuracy deltas across all conditions as
a heatmap.
Table~\ref{tab:mcnemar500} reports McNemar tests at $n \approx 500$ as an
intermediate power check.

\begin{table}[htbp]
	\caption{Answer log-probability (mean), $n=200$.}
	\label{tab:logprob}
	\centering
	\small
	\begin{tabular}{llcccc}
		\toprule
		Model & Dataset  & C1       & C2       & C3       & C4       \\
		\midrule
		\multirow{3}{*}{Mistral-7B}
		      & CSQA     & $-0.005$ & $-0.001$ & $-0.002$ & $-0.005$ \\
		      & HotpotQA & $-0.635$ & $-0.569$ & $-0.603$ & $-0.535$ \\
		      & TriviaQA & $-0.201$ & $-0.212$ & $-0.196$ & $-0.186$ \\
		\midrule
		\multirow{3}{*}{LLaMA-3-8B}
		      & CSQA     & $-0.002$ & $-0.004$ & $-0.003$ & $-0.002$ \\
		      & HotpotQA & $-0.994$ & $-1.142$ & $-1.096$ & $-1.113$ \\
		      & TriviaQA & $-0.331$ & $-0.384$ & $-0.340$ & $-0.329$ \\
		\bottomrule
	\end{tabular}
\end{table}

\begin{table}[htbp]
	\caption{Per-sample error flow (C1 $\to$ C2), $n{=}200$.}
	\label{tab:errorflow}
	\centering
	\small
	\begin{tabular}{llccccc}
		\toprule
		Model & Dataset  & C1 Err. & C2 Err. & Fixed & Broken & Net               \\
		\midrule
		\multirow{3}{*}{Mistral-7B}
		      & CSQA     & 60      & 64      & 14    & 18     & $-4$ ($-2.0$\,pp) \\
		      & HotpotQA & 156     & 151     & 8     & 3      & $+5$ ($+2.5$\,pp) \\
		      & TriviaQA & 56      & 54      & 7     & 5      & $+2$ ($+1.0$\,pp) \\
		\midrule
		\multirow{3}{*}{LLaMA-3-8B}
		      & CSQA     & 54      & 56      & 13    & 15     & $-2$ ($-1.0$\,pp) \\
		      & HotpotQA & 151     & 147     & 6     & 2      & $+4$ ($+2.0$\,pp) \\
		      & TriviaQA & 50      & 59      & 1     & 10     & $-9$ ($-4.5$\,pp) \\
		\bottomrule
	\end{tabular}
\end{table}

\begin{figure}[htbp]
	\centering
	\begin{tikzpicture}
		\begin{axis}[
				hm colormap,
				colorbar,
				colorbar style={
						width=8pt,
						ytick={-5,-2.5,0,2.5,5},
						yticklabels={$-5$,$-2.5$,$0$,$+2.5$,$+5$},
						yticklabel style={font=\scriptsize},
						ylabel={$\Delta$ acc.\ (pp)},
						ylabel style={font=\scriptsize, yshift=2pt},
					},
				point meta min=-5, point meta max=5,
				width=0.82\textwidth, height=7.8cm,
				xmin=0.5, xmax=6.5,
				ymin=0.5, ymax=5.5,
				enlargelimits=false,
				axis on top,
				xtick={1,2,3,4,5,6},
				xticklabels={CSQA, HotpotQA, TriviaQA, CSQA, HotpotQA, TriviaQA},
				xtick align=outside,
				xticklabel style={rotate=35, anchor=north east, font=\scriptsize,
						yshift=-2pt, xshift=2pt},
				extra x ticks={1.5,2.5,3.5,4.5,5.5},
				extra x tick labels={},
				extra x tick style={grid=major, grid style={white, line width=2pt}},
				ytick={1,2,3,4,5},
				yticklabels={\scriptsize C6 Multi-Feat,
						\scriptsize C5 Neutral,
						\scriptsize C4 Scalar Gate,
						\scriptsize C3 Noise KG,
						\scriptsize C2 Rel.~KG},
				yticklabel style={font=\scriptsize\bfseries},
				extra y ticks={1.5,2.5,3.5,4.5},
				extra y tick labels={},
				extra y tick style={grid=major, grid style={white, line width=2pt}},
				tick style={draw=none},
				clip=false,
			]
			\addplot[matrix plot*, mesh/cols=6, point meta=explicit]
			table[x=x, y=y, meta=val] {\hmdata};

			\pgfplotsextra{%
				\pgfplotstableforeachcolumnelement{val}\of\hmdata\as\val{%
					\pgfplotstablegetelem{\pgfplotstablerow}{x}\of{\hmdata}
					\let\xv=\pgfplotsretval
					\pgfplotstablegetelem{\pgfplotstablerow}{y}\of{\hmdata}
					\let\yv=\pgfplotsretval
					\pgfmathsetmacro{\absval}{abs(\val)}
					\ifdim\absval pt > 2.5pt
						\def\tcolor{white}
					\else
						\def\tcolor{black!70}
					\fi
					\node[font=\scriptsize\bfseries, text=\tcolor]
					at (axis cs:\xv,\yv)
					{\pgfmathprintnumber[print sign, fixed, precision=1]{\val}};
				}%
				\draw[blue!70!black, line width=1.5pt]
				(axis cs:4.5, 4.5) rectangle (axis cs:5.5, 5.5);
				\node[font=\tiny\bfseries, text=blue!70!black, anchor=south east]
				at (axis cs:5.42, 5.42) {$\dagger$};
				\node[font=\small\bfseries, anchor=south]
				at (axis cs:2.0, 5.7) {Mistral-7B};
				\node[font=\small\bfseries, anchor=south]
				at (axis cs:5.0, 5.7) {LLaMA-3-8B};
				\draw[gray!50, thin] (axis cs:3.5,5.55) -- (axis cs:3.5,5.7);
				\draw[gray!60, dashed, thin]
				(axis cs:3.5,0.5) -- (axis cs:3.5,5.5);
			}%
		\end{axis}
	\end{tikzpicture}
	\caption{KG-gain heatmap: accuracy delta (pp) vs.\ C1.
	$\dagger$: only Bonferroni-significant result at $n{=}1{,}000$.}
	\label{fig:heatmap}
\end{figure}
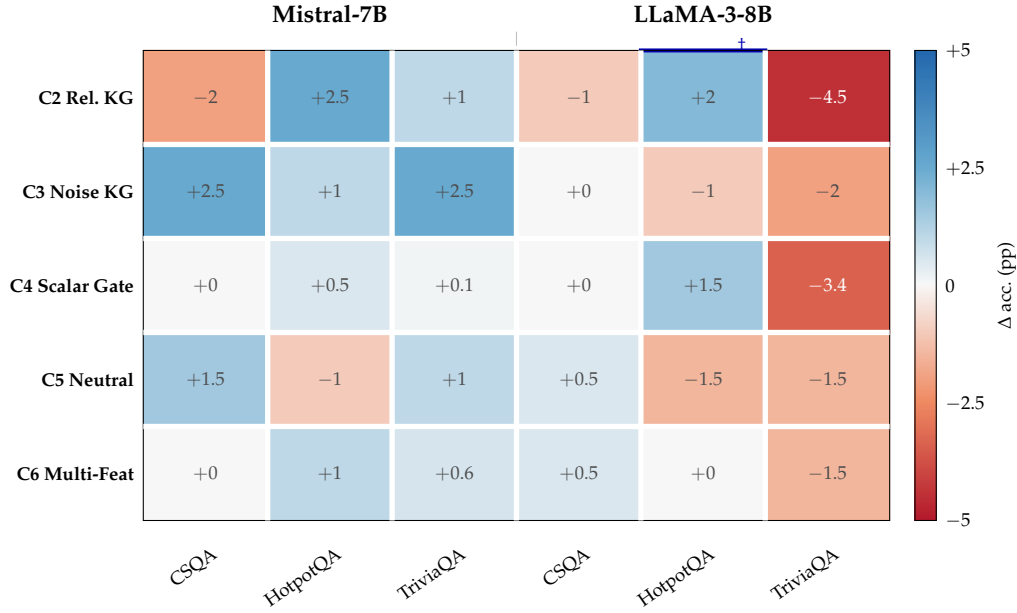

\begin{table}[htbp]
	\caption{McNemar at $n \approx 500$.
		$\dagger$ = $p_{\text{Bonf}} < 0.05$.}
	\label{tab:mcnemar500}
	\centering
	\small
	\setlength{\tabcolsep}{2.6pt}
	\begin{tabular}{llcccccc}
		\toprule
		Model & Dataset          & $n$               & $n_{\text{fix}}$ & $n_{\text{brk}}$
		      & $p_{\text{raw}}$ & $p_{\text{Bonf}}$ & OR                                                                                       \\
		\midrule
		\multirow{3}{*}{Mistral-7B}
		      & CSQA             & 500               & 24               & 41               & .046            & .278                      & 0.59 \\
		      & HotpotQA         & 500               & 18               & 14               & .597            & 1.00                      & 1.29 \\
		      & TriviaQA         & 491               & 13               & 25               & .073            & .438                      & 0.52 \\
		\multirow{3}{*}{LLaMA-3-8B}
		      & CSQA             & 500               & 25               & 38               & .130            & .779                      & 0.66 \\
		      & HotpotQA         & 500               & 22               & 5                & $\mathbf{.002}$ & $\mathbf{.009}^{\dagger}$ & 4.40 \\
		      & TriviaQA         & 491               & 8                & 24               & $\mathbf{.007}$ & $\mathbf{.042}^{\dagger}$ & 0.33 \\
		\bottomrule
	\end{tabular}
\end{table}

\subsection{Neutral-Text Decomposition (Full)}
\label{app:c5_full}

Table~\ref{tab:c5_decomp} reports the full mechanism decomposition via C5,
separating the C2 effect into a crowding component (C1$-$C5, attributable to
prompt lengthening) and a semantic component (C5$-$C2, attributable to
KG-specific content).

\begin{table}[htbp]
	\caption{Mechanism decomposition via C5, $n{=}200$.}
	\label{tab:c5_decomp}
	\centering
	\small
	\setlength{\tabcolsep}{3.5pt}
	\begin{tabular}{llccccc}
		\toprule
		Model & Dataset  & C1   & C2   & C5   & Crowd. & Sem.             \\
		\midrule
		\multirow{3}{*}{Mistral-7B}
		      & CSQA     & 70.0 & 68.0 & 71.5 & $-1.5$ & $+3.5$           \\
		      & HotpotQA & 22.0 & 24.5 & 21.0 & $+1.0$ & $-3.5$           \\
		      & TriviaQA & 71.5 & 72.5 & 72.5 & $-1.0$ & $\phantom{+}0.0$ \\
		\midrule
		\multirow{3}{*}{LLaMA-3-8B}
		      & CSQA     & 73.0 & 72.0 & 73.5 & $-0.5$ & $+1.5$           \\
		      & HotpotQA & 24.5 & 26.5 & 23.0 & $+1.5$ & $-3.5$           \\
		      & TriviaQA & 74.5 & 70.0 & 73.1 & $+1.5$ & $+3.0$           \\
		\bottomrule
	\end{tabular}
\end{table}

\subsection{Multi-Feature Gating (C6) and Oracle Analysis}
\label{app:oracle}

Table~\ref{tab:oracle} reports C6 gating results alongside an oracle upper
bound.
C6 improves over C4 by at most 2\,pp in any setting, while the oracle
(which injects KG only when it helps) shows 3--7\,pp headroom---indicating
that better gating could help in principle, but source selection
(Section~\ref{sec:bm25}) remains a far larger lever.

\begin{table}[htbp]
	\caption{C6 results and oracle upper bound.}
	\label{tab:oracle}
	\centering
	\small
	\setlength{\tabcolsep}{2.8pt}
	\begin{tabular}{llcccccccc}
		\toprule
		Model & Dataset  & C1   & C2   & C4   & C6   & C6$-$C4 & Rate & Oracle & Orc.$-$C1 \\
		\midrule
		\multirow{3}{*}{Mistral}
		      & CSQA     & 70.0 & 68.0 & 70.0 & 70.0 & $\pm$0  & 1\%  & 77.0   & +7.0      \\
		      & HotpotQA & 22.0 & 24.5 & 22.5 & 23.0 & +0.5    & 49\% & 26.0   & +4.0      \\
		      & TriviaQA & 71.5 & 72.5 & 71.6 & 72.1 & +0.5    & 29\% & 75.1   & +3.6      \\
		\midrule
		\multirow{3}{*}{LLaMA}
		      & CSQA     & 73.0 & 72.0 & 73.0 & 73.5 & +0.5    & 1\%  & 79.5   & +6.5      \\
		      & HotpotQA & 24.5 & 26.5 & 26.0 & 24.5 & $-$1.5  & 16\% & 27.5   & +3.0      \\
		      & TriviaQA & 74.5 & 70.0 & 71.1 & 73.1 & +2.0    & 23\% & 75.1   & +0.6      \\
		\bottomrule
	\end{tabular}
\end{table}

\section{Attention Analysis}
\label{app:attention_combined}

\subsection{Attention Analysis}
\label{app:attention}

Tables~\ref{tab:attention} and~\ref{tab:attention_llama} report last-layer
attention distributions for the first answer token.
Key observations: (i)~demonstration attention drops substantially when KG
is present (from $\sim$23\% to $\sim$13\% for Mistral), (ii)~KG-broken
and KG-fixed samples show similar attention patterns despite opposite
accuracy outcomes, directly supporting Prediction~4.

\begin{table}[htbp]
	\caption{Attention (Mistral-7B, last layer, first answer token).}
	\label{tab:attention}
	\centering
	\small
	\setlength{\tabcolsep}{3.5pt}
	\begin{tabular}{lllcccc}
		\toprule
		Type & Dataset                   & Cond. & Instr.        & Demo & KG   & Question \\
		\midrule
		\multirow{9}{*}{\rotatebox{90}{KG-broken}}
		     & \multirow{3}{*}{CSQA}
		     & C1                        & 6.2   & \textbf{22.8} & ---  & 36.4            \\
		     &                           & C2    & 5.4           & 13.3 & 10.3 & 40.1     \\
		     &                           & C3    & 5.3           & 14.0 & 7.7  & 41.8     \\
		\cmidrule{2-7}
		     & \multirow{2}{*}{HotpotQA}
		     & C1                        & 3.1   & \textbf{15.2} & ---  & 44.1            \\
		     &                           & C2    & 3.4           & 11.8 & 8.6  & 41.8     \\
		\cmidrule{2-7}
		     & \multirow{2}{*}{TriviaQA}
		     & C1                        & 3.4   & \textbf{13.5} & ---  & 42.5            \\
		     &                           & C2    & 3.1           & 12.0 & 8.1  & 42.2     \\
		\midrule
		\multirow{6}{*}{\rotatebox{90}{KG-fixed}}
		     & \multirow{2}{*}{CSQA}
		     & C1                        & 6.1   & \textbf{23.5} & ---  & 35.1            \\
		     &                           & C2    & 5.6           & 13.4 & 9.6  & 38.6     \\
		\cmidrule{2-7}
		     & \multirow{2}{*}{HotpotQA}
		     & C1                        & 3.0   & \textbf{16.0} & ---  & 42.8            \\
		     &                           & C2    & 3.1           & 13.5 & 7.8  & 41.7     \\
		\cmidrule{2-7}
		     & \multirow{2}{*}{TriviaQA}
		     & C1                        & 3.5   & \textbf{16.4} & ---  & 42.3            \\
		     &                           & C2    & 3.3           & 13.6 & 7.2  & 44.2     \\
		\bottomrule
	\end{tabular}
\end{table}

\begin{table}[htbp]
	\caption{Attention (LLaMA-3-8B, last layer, KG-broken).}
	\label{tab:attention_llama}
	\centering
	\small
	\setlength{\tabcolsep}{3.0pt}
	\begin{tabular}{llcccc}
		\toprule
		Dataset & Cond. & Instr. & Demo         & KG  & Ques. \\
		\midrule
		\multirow{3}{*}{CSQA}
		        & C1    & 0.8    & \textbf{8.4} & --- & 32.8  \\
		        & C2    & 0.7    & 6.4          & 6.4 & 33.4  \\
		        & C3    & 0.6    & 6.5          & 4.6 & 34.3  \\
		\midrule
		\multirow{2}{*}{HotpotQA}
		        & C1    & 0.5    & \textbf{6.1} & --- & 30.3  \\
		        & C2    & 0.4    & 5.0          & 3.0 & 31.5  \\
		\midrule
		\multirow{2}{*}{TriviaQA}
		        & C1    & 0.4    & \textbf{6.6} & --- & 30.7  \\
		        & C2    & 0.5    & 5.4          & 4.9 & 32.8  \\
		\bottomrule
	\end{tabular}
\end{table}

\subsection{Multi-Layer Attention Analysis}
\label{app:multilayer}

The main-text attention analysis uses the last transformer layer only.
Here we report attention distributions at four evenly spaced layers
(0, 10, 20, and the final layer) for both models under C1, C2, and C3,
confirming that the structural attention tax is not an artefact of
last-layer dynamics.

\begin{table}[htbp]
	\caption{Multi-layer attention (\%) for Mistral-7B on CSQA (KG-broken
		samples). KG attention is consistently elevated under C2 and C3 across
		all layers.}
	\label{tab:multilayer_mistral}
	\centering
	\small
	\setlength{\tabcolsep}{3pt}
	\begin{tabular}{llcccc}
		\toprule
		Layer & Cond. & Instr. & Demo          & KG   & Ques. \\
		\midrule
		\multirow{3}{*}{0}
		      & C1    & 6.0    & \textbf{19.4} & ---  & 23.2  \\
		      & C2    & 5.2    & 14.0          & 7.6  & 22.5  \\
		      & C3    & 5.1    & 14.0          & 8.0  & 22.4  \\
		\midrule
		\multirow{3}{*}{10}
		      & C1    & 8.3    & \textbf{13.6} & ---  & 37.7  \\
		      & C2    & 7.5    & 6.5           & 6.7  & 39.1  \\
		      & C3    & 7.4    & 6.2           & 5.4  & 40.3  \\
		\midrule
		\multirow{3}{*}{20}
		      & C1    & 9.6    & \textbf{13.2} & ---  & 17.3  \\
		      & C2    & 7.1    & 11.3          & 6.9  & 18.9  \\
		      & C3    & 7.7    & 11.6          & 3.6  & 20.0  \\
		\midrule
		\multirow{3}{*}{31 (last)}
		      & C1    & 6.1    & \textbf{22.2} & ---  & 37.1  \\
		      & C2    & 5.3    & 12.8          & 10.5 & 40.1  \\
		      & C3    & 5.3    & 13.6          & 7.6  & 42.1  \\
		\bottomrule
	\end{tabular}
\end{table}

\begin{table}[htbp]
	\caption{Multi-layer attention (\%) for LLaMA-3-8B on CSQA (KG-broken
		samples).}
	\label{tab:multilayer_llama}
	\centering
	\small
	\setlength{\tabcolsep}{3pt}
	\begin{tabular}{llcccc}
		\toprule
		Layer & Cond. & Instr. & Demo          & KG  & Ques. \\
		\midrule
		\multirow{3}{*}{0}
		      & C1    & 3.1    & \textbf{23.1} & --- & 45.3  \\
		      & C2    & 2.4    & 14.8          & 9.6 & 45.5  \\
		      & C3    & 2.5    & 15.0          & 9.3 & 45.6  \\
		\midrule
		\multirow{3}{*}{10}
		      & C1    & 2.3    & \textbf{27.5} & --- & 28.3  \\
		      & C2    & 1.2    & 20.9          & 8.1 & 29.5  \\
		      & C3    & 1.4    & 21.7          & 4.5 & 30.0  \\
		\midrule
		\multirow{3}{*}{20}
		      & C1    & 0.7    & \textbf{6.8}  & --- & 16.3  \\
		      & C2    & 0.7    & 5.8           & 7.5 & 15.4  \\
		      & C3    & 0.7    & 6.0           & 2.6 & 17.2  \\
		\midrule
		\multirow{3}{*}{31 (last)}
		      & C1    & 0.8    & \textbf{8.2}  & --- & 33.2  \\
		      & C2    & 0.7    & 6.3           & 6.0 & 33.8  \\
		      & C3    & 0.7    & 6.3           & 4.2 & 34.7  \\
		\bottomrule
	\end{tabular}
\end{table}

Key observations: (i)~KG attention elevation appears from the earliest
layers (layer~0), indicating that the structural tax is established
during initial token processing, not solely a late-layer phenomenon;
(ii)~demonstration compression under C2/C3 is visible at every layer;
(iii)~the C2/C3 similarity (format-invariant capture) holds across
layers, strengthening the case that the structural tax is
format-driven rather than content-driven.

Figures~\ref{fig:multilayer_mistral} and~\ref{fig:multilayer_llama}
provide visual summaries.

\begin{figure}[htbp]
	\centering
	\includegraphics[width=\columnwidth]{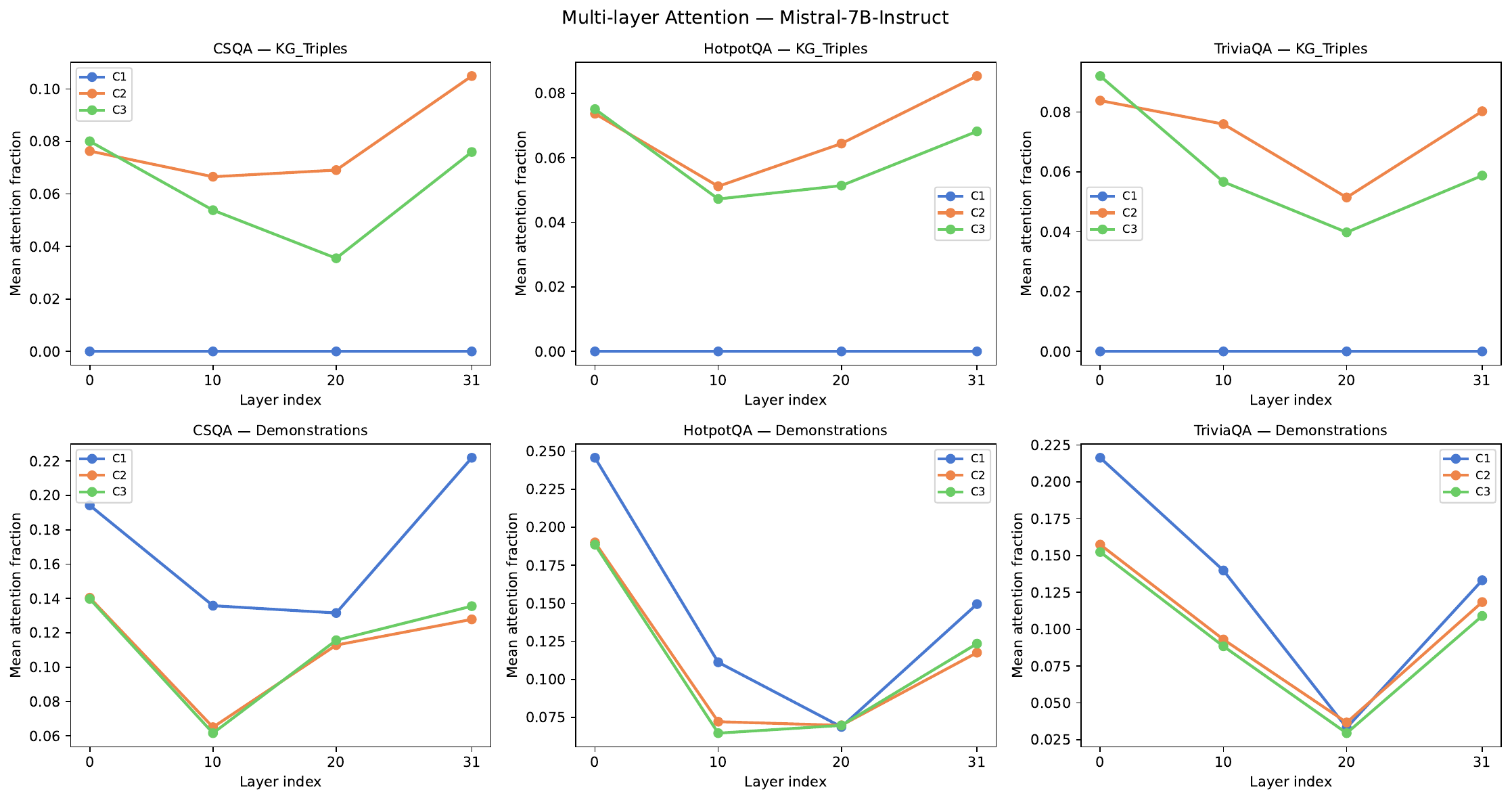}
	\caption{Multi-layer attention distribution for Mistral-7B.}
	\label{fig:multilayer_mistral}
\end{figure}

\begin{figure}[htbp]
	\centering
	\includegraphics[width=\columnwidth]{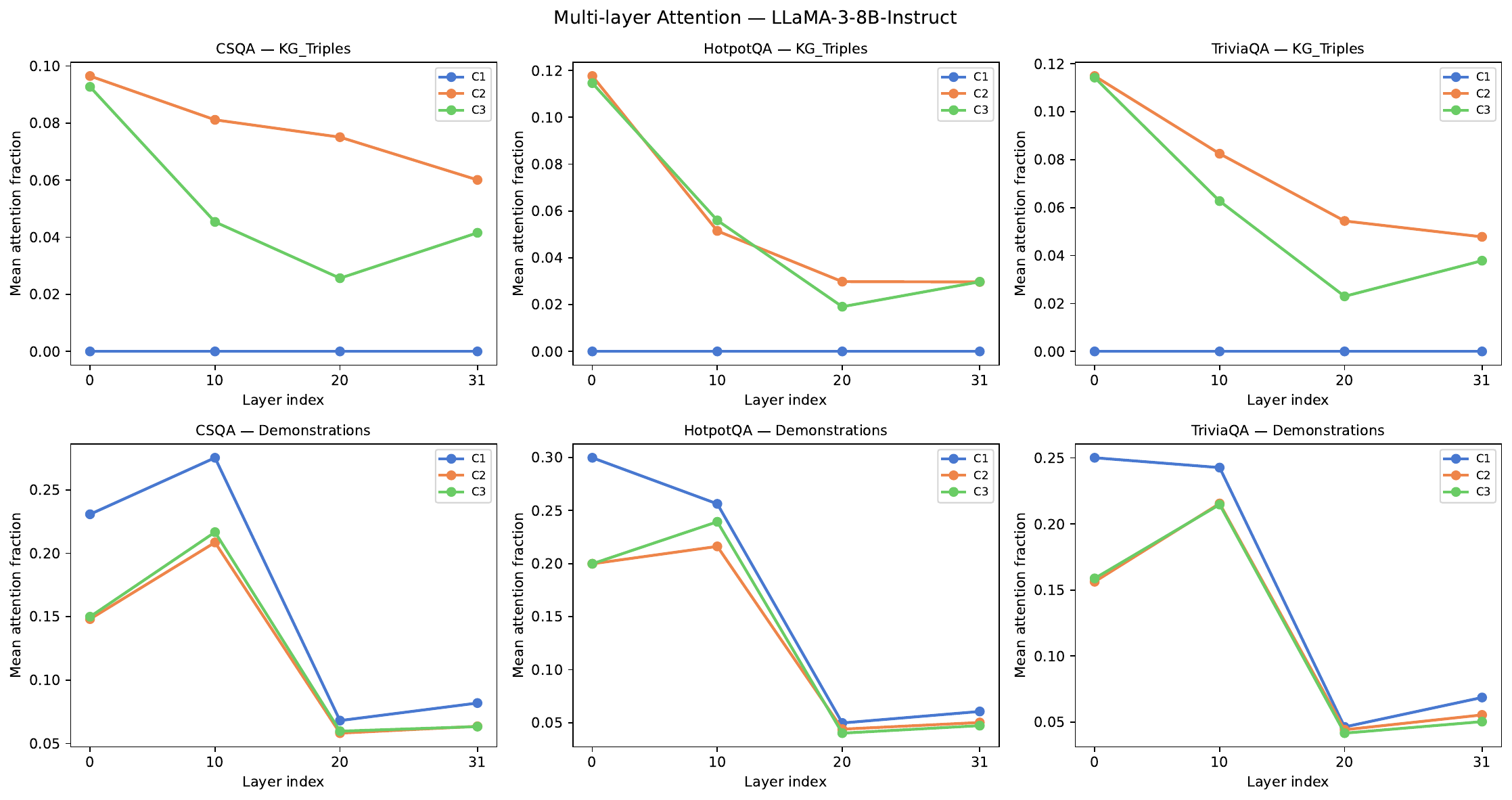}
	\caption{Multi-layer attention distribution for LLaMA-3-8B.}
	\label{fig:multilayer_llama}
\end{figure}

\section{Ablation and Robustness}
\label{app:ablation}

\subsection{$k$-Shot Ablation}
\label{app:kablation}

Table~\ref{tab:kshot} reports accuracy under C2 as the number of
demonstrations varies from $k{=}1$ to $k{=}7$.
Performance is relatively stable across $k$ values, suggesting that
the structural attention tax scales proportionally with demonstration
count rather than exhibiting sharp threshold effects.

\begin{table}[htbp]
	\caption{$k$-shot ablation: accuracy (\%) under C2.}
	\label{tab:kshot}
	\centering
	\small
	\begin{tabular}{llcccc}
		\toprule
		Model & Dataset  & $k=1$ & $k=3$ & $k=5$ & $k=7$ \\
		\midrule
		\multirow{3}{*}{Mistral-7B}
		      & CSQA     & 68.5  & 68.0  & 69.0  & 68.0  \\
		      & HotpotQA & 23.5  & 24.5  & 25.0  & 25.5  \\
		      & TriviaQA & 72.5  & 72.5  & 69.5  & 70.0  \\
		\midrule
		\multirow{3}{*}{LLaMA-3-8B}
		      & CSQA     & 69.0  & 72.0  & 71.0  & 70.0  \\
		      & HotpotQA & 24.0  & 26.5  & 25.0  & 24.5  \\
		      & TriviaQA & 71.0  & 70.0  & 72.0  & 71.5  \\
		\bottomrule
	\end{tabular}
\end{table}

\subsection{Noise Condition Validation}
\label{app:noise}

Table~\ref{tab:noise} confirms that C3 triples are genuinely unrelated to the
query: mean cosine similarity is near zero for C3 across all datasets,
compared with moderate similarity for C2.
The large Mann--Whitney effect sizes ($r = 0.37$--$0.76$) confirm effective
separation between conditions, validating the C2/C3 contrast as a clean test
of Prediction~3.

\begin{table}[htbp]
	\caption{Noise validation: cosine similarity (C2 vs.\ C3).}
	\label{tab:noise}
	\centering
	\small
	\begin{tabular}{lcccc}
		\toprule
		Dataset  & C2 sim.         & C3 sim.         & MW $p$                & $r$   \\
		\midrule
		CSQA     & $.105 \pm .090$ & $.004 \pm .008$ & $1.5 \times 10^{-43}$ & 0.760 \\
		HotpotQA & $.135 \pm .164$ & $.018 \pm .034$ & $7.8 \times 10^{-12}$ & 0.369 \\
		TriviaQA & $.118 \pm .147$ & $.020 \pm .043$ & $4.9 \times 10^{-13}$ & 0.392 \\
		\bottomrule
	\end{tabular}
\end{table}

\subsection{Difficulty-Stratified Analysis}
\label{app:difficulty}

Tables~\ref{tab:diff_mistral}--\ref{tab:diff_indep} report accuracy stratified
by question difficulty.
KG harm concentrates in the Complex tier under independent-feature
stratification, consistent with the structural attention tax hypothesis:
complex questions rely more heavily on demonstration attention, making them
more vulnerable to format-driven compression.

\begin{table}[htbp]
	\caption{Difficulty-stratified accuracy (\%) for Mistral-7B.}
	\label{tab:diff_mistral}
	\centering
	\small
	\setlength{\tabcolsep}{4.5pt}
	\begin{tabular}{llcccc}
		\toprule
		Dataset & Tier   & C1                                  & C2             & C3             & C4             \\
		\midrule
		\multirow{3}{*}{CSQA}
		        & Hard   & 57.58                               & 56.06          & \textbf{62.12} & 57.58          \\
		        & Medium & 71.21                               & 69.70          & \textbf{72.73} & 71.21          \\
		        & Easy   & 80.88                               & 77.94          & \textbf{82.35} & 80.88          \\
		\midrule
		\multirow{3}{*}{HotpotQA}
		        & Hard   & \multicolumn{4}{c}{9.09 (all tied)}                                                    \\
		        & Medium & 16.67                               & \textbf{18.18} & \textbf{18.18} & \textbf{18.18} \\
		        & Easy   & 39.71                               & \textbf{45.59} & 41.18          & 39.71          \\
		\midrule
		\multirow{3}{*}{TriviaQA}
		        & Hard   & 55.38                               & 56.92          & \textbf{61.54} & 55.38          \\
		        & Medium & 70.77                               & \textbf{72.31} & 70.77          & 70.77          \\
		        & Easy   & 88.06                               & 88.06          & \textbf{89.55} & 88.06          \\
		\bottomrule
	\end{tabular}
\end{table}

\begin{table}[htbp]
	\caption{Difficulty-stratified accuracy (\%) for LLaMA-3-8B.}
	\label{tab:diff_llama}
	\centering
	\small
	\setlength{\tabcolsep}{4.5pt}
	\begin{tabular}{llcccc}
		\toprule
		Dataset & Tier   & C1                                   & C2             & C3             & C4             \\
		\midrule
		\multirow{3}{*}{CSQA}
		        & Hard   & 60.61                                & 63.64          & \textbf{66.67} & 60.61          \\
		        & Medium & \textbf{72.73}                       & 63.64          & 69.70          & \textbf{72.73} \\
		        & Easy   & 85.29                                & \textbf{88.24} & 82.35          & 85.29          \\
		\midrule
		\multirow{3}{*}{HotpotQA}
		        & Hard   & \multicolumn{4}{c}{6.06 (all tied)}                                                     \\
		        & Medium & 22.73                                & \textbf{24.24} & 19.70          & \textbf{24.24} \\
		        & Easy   & 44.12                                & \textbf{48.53} & 44.12          & 47.06          \\
		\midrule
		\multirow{3}{*}{TriviaQA}
		        & Hard   & \textbf{44.62}                       & 33.85          & 40.00          & 33.85          \\
		        & Medium & \textbf{89.23}                       & 86.15          & 87.69          & \textbf{89.23} \\
		        & Easy   & \multicolumn{4}{c}{89.55 (all tied)}                                                    \\
		\bottomrule
	\end{tabular}
\end{table}

\begin{table}[htbp]
	\caption{Independent-feature stratification for LLaMA-3-8B.}
	\label{tab:diff_indep}
	\centering
	\small
	\setlength{\tabcolsep}{4.5pt}
	\begin{tabular}{llccccc}
		\toprule
		Dataset & Tier    & $n$ & C1            & C2            & C3            & C4            \\
		\midrule
		\multirow{3}{*}{CSQA}
		        & Simple  & 66  & 68.2          & 72.7          & \textbf{75.8} & 68.2          \\
		        & Medium  & 66  & \textbf{74.2} & \textbf{77.3} & 71.2          & \textbf{74.2} \\
		        & Complex & 68  & \textbf{76.5} & 66.2          & 72.1          & \textbf{76.5} \\
		\midrule
		\multirow{3}{*}{TriviaQA}
		        & Simple  & 65  & \textbf{78.5} & 75.4          & 75.4          & 76.9          \\
		        & Medium  & 65  & \textbf{84.6} & 80.0          & 83.1          & 80.0          \\
		        & Complex & 67  & \textbf{61.2} & 55.2          & 59.7          & 56.7          \\
		\bottomrule
	\end{tabular}
\end{table}

\section{Mitigation Strategy Experiments}
\label{app:mitigation}

\subsection{Verbalized-Triple Control (C5b)}
\label{app:c5b}

C5b converts C2 triples to natural language, testing
Strategy~S3 (format flattening).
Table~\ref{tab:c5b} reports accuracy across models and datasets.

\begin{table}[htbp]
	\caption{C5b accuracy (\%), $n{=}200$.}
	\label{tab:c5b}
	\centering
	\small
	\begin{tabular}{llcccc}
		\toprule
		Model & Dataset  & C1   & C2   & C5   & C5b  \\
		\midrule
		\multirow{3}{*}{Mistral-7B}
		      & CSQA     & 70.0 & 68.0 & 71.5 & 70.0 \\
		      & HotpotQA & 22.0 & 24.5 & 21.0 & 26.0 \\
		      & TriviaQA & 71.5 & 72.5 & 72.5 & 73.0 \\
		\midrule
		\multirow{3}{*}{LLaMA-3-8B}
		      & CSQA     & 73.0 & 72.0 & 73.5 & 72.0 \\
		      & HotpotQA & 24.5 & 26.5 & 23.0 & 24.5 \\
		      & TriviaQA & 74.5 & 70.0 & 73.0 & 71.5 \\
		\bottomrule
	\end{tabular}
\end{table}

Beyond accuracy, we compare the attention distributions of C2 and C5b
to test whether verbalization reduces the structural attention tax.
Table~\ref{tab:c5b_attn} reports KG-region and demonstration-region
attention for both conditions.

\begin{table}[htbp]
	\caption{C2 vs.\ C5b attention comparison (last layer, \%).
		Ratio $=$ C5b KG attn / C2 KG attn; values $< 1.00$ indicate that
		verbalization reduces KG-region attention capture.}
	\label{tab:c5b_attn}
	\centering
	\small
	\setlength{\tabcolsep}{3pt}
	\begin{tabular}{llccccc}
		\toprule
		Model & Dataset  & C2 $A_K$ & C5b $A_K$ & Ratio         & C2 $A_D$ & C5b $A_D$ \\
		\midrule
		\multirow{3}{*}{Mistral-7B}
		      & CSQA     & 9.9      & 9.4       & 0.95          & 13.3     & 14.0      \\
		      & HotpotQA & 8.2      & 7.4       & 0.90          & 12.7     & 11.6      \\
		      & TriviaQA & 7.7      & 8.7       & 1.14          & 12.8     & 11.4      \\
		\midrule
		\multirow{3}{*}{LLaMA-3-8B}
		      & CSQA     & 6.4      & 5.0       & \textbf{0.79} & 6.6      & 6.8       \\
		      & HotpotQA & 3.5      & 2.9       & \textbf{0.83} & 4.7      & 4.7       \\
		      & TriviaQA & 5.0      & 3.6       & \textbf{0.71} & 5.6      & 5.4       \\
		\bottomrule
	\end{tabular}
\end{table}

\textbf{LLaMA-3-8B} shows a consistent reduction in KG-region
attention across all three datasets (ratio $0.71$--$0.83$,
corresponding to $17$--$29\%$ less KG attention), while demonstration
attention remains stable ($\Delta \leq 0.2$\,pp).
This confirms that format flattening reduces the structural tax as
predicted by Strategy~S3: removing delimiter and slot patterns lowers
$\sigma(K)$ without disrupting semantic signal transmission.
\textbf{Mistral-7B} shows weaker and less consistent effects
(ratio $0.90$--$1.14$), with TriviaQA showing a slight increase,
suggesting that Mistral is less sensitive to delimiter-driven
structural capture than LLaMA.

Combining accuracy evidence (4/6 pairs where C5b $\geq$ C1) with
attention evidence (consistent reduction for LLaMA-3-8B), C5b provides
direct support for S3 as an effective zero-cost mitigation strategy,
though the effect magnitude varies across model architectures.

\subsection{Structural Dispersal Experiment (S1)}
\label{app:s1_dispersal}

We test Strategy~S1 (Structural Dispersal) by interleaving each KG
triple with a natural-language bridging phrase (e.g., ``Consider
that\ldots'', ``Additionally,\ldots''), replacing the contiguous
triple block with sentence-level separation.

\paragraph{Design.}
Each triple in the C2 prompt is wrapped with a randomly sampled
bridging phrase from a pool of seven options (``Consider that'',
``Additionally,'', ``Furthermore,'', ``Also note that'',
``It is known that'', ``Moreover,'', ``Note also that'') and
terminated with a period.
All other prompt components (instruction, demonstrations, question)
remain identical.
We run inference on the same 200 samples per dataset and compare
accuracy and last-layer attention distributions.

\paragraph{Hypothesis.}
If the structural attention tax is driven by contiguous delimiter
patterns and slot repetitiveness, dispersal should reduce KG-region
attention capture and partially restore demonstration attention.

\paragraph{Results.}
Table~\ref{tab:s1} reports accuracy and KG-region attention for C2
vs.\ S1 across both models and three datasets.

\begin{table}[htbp]
	\caption{S1 Dispersal results: accuracy (\%) and KG attention (\%).
		Ratio $=$ S1 $A_K$ / C2 $A_K$; $\uparrow$ indicates attention
		\emph{increased} (opposite to prediction), $\downarrow$ indicates the
		predicted decrease.}
	\label{tab:s1}
	\centering
	\small
	\setlength{\tabcolsep}{2.5pt}
	\begin{tabular}{llcccccc}
		\toprule
		Model & Dataset  & C2 Acc & S1 Acc & $\Delta$ & C2 $A_K$ & S1 $A_K$ & Ratio            \\
		\midrule
		\multirow{3}{*}{Mistral-7B}
		      & CSQA     & 68.0   & 67.5   & $-0.5$   & 9.9      & 11.4     & 1.15$\uparrow$   \\
		      & HotpotQA & 24.5   & 25.5   & $+1.0$   & 8.2      & 10.2     & 1.25$\uparrow$   \\
		      & TriviaQA & 72.5   & 66.0   & $-6.5$   & 7.7      & 12.2     & 1.59$\uparrow$   \\
		\midrule
		\multirow{3}{*}{LLaMA-3-8B}
		      & CSQA     & 72.0   & 71.0   & $-1.0$   & 6.4      & 3.9      & 0.61$\downarrow$ \\
		      & HotpotQA & 26.5   & 25.0   & $-1.5$   & 3.5      & 2.5      & 0.71$\downarrow$ \\
		      & TriviaQA & 70.0   & 65.5   & $-4.5$   & 5.0      & 4.0      & 0.80$\downarrow$ \\
		\bottomrule
	\end{tabular}
\end{table}

\paragraph{Analysis.}
The results reveal a striking model dependence:

\textbf{LLaMA-3-8B} shows the predicted effect direction: S1 dispersal
reduces KG-region attention by 20--39\% (ratio $0.61$--$0.80$),
confirming that contiguous structural density contributes to attention
capture in this model.
However, accuracy simultaneously degrades by $1.0$--$4.5$\,pp,
suggesting that the bridging phrases, while reducing structural
density, also interfere with the semantic signal carried by the triples.

\textbf{Mistral-7B} shows a paradoxical \emph{increase} in KG-region
attention under dispersal (ratio $1.15$--$1.59$), with the largest
increase on TriviaQA ($+59\%$) accompanied by a substantial accuracy
drop ($-6.5$\,pp).
A plausible explanation is that bridging phrases such as ``Consider
that'' and ``Furthermore'' themselves contain sentence-level structural
cues that function as new attention anchors in Mistral's attention
mechanism, effectively adding structural salience rather than
dispersing it.

\paragraph{Implications.}
The contrasting outcomes of S1 (mixed/negative) and S3 (supported;
Appendix~\ref{app:c5b}) are informative: both target $\sigma(K)$,
but S3 \emph{eliminates} structural patterns by converting to natural
prose, while S1 \emph{dilutes} them by interleaving with additional
tokens.
The S1 results suggest that the choice of bridging material matters
critically---tokens with their own semantic weight (verbs like
``Consider'', discourse markers like ``Furthermore'') can become
attention attractors rather than attention dispersers.
This highlights that effective structural tax mitigation requires
genuine format transformation (as in S3), not merely structural
dilution.

\section{Theoretical Derivations}
\label{app:theory_combined}

\subsection{Detailed Theoretical Derivations}
\label{app:theory_details}

\subsubsection{Three-Regime Decomposition}

Partition queries into three regimes based on $c_0(q)$:
\textbf{Regime I} ($c_0 > 1 - \epsilon$): Knowledge can only
redistribute mass away from $y^*$.
\textbf{Regime II} ($\epsilon < c_0 < 1 - \epsilon$): Both improvement
and degradation possible.
\textbf{Regime III} ($c_0 < \epsilon$): Near-random baseline; relevant
$K$ has high marginal value.
For CSQA, $>$70\% of samples are in Regime~I; for HotpotQA,
Regimes~II--III dominate.

\subsubsection{Proof of Proposition 1}

\begin{proof}
	Let $Z_0 = \sum_{k \notin K} \exp(s_{ik})$ and
	$Z_K = \sum_{j \in K} \exp(s_{ij})$.
	Then $A_D^{(K)} = A_D^{(0)} \cdot Z_0 / (Z_0 + Z_K)$.
	By Jensen's inequality, $Z_K \leq m \cdot \exp(\bar{s}_K)$;
	$Z_0 \geq T_0 \cdot \exp(\bar{s}_D)$, giving the result.
\end{proof}

\subsubsection{Compression Bound Calibration}
\label{app:bound_calibration}

For Mistral-7B on CSQA: $A_D^{(K)}/A_D^{(0)} = 13.3/22.8 = 0.58$.
With $T_0 \approx 350$, $m \approx 30$: $\bar{s}_K - \bar{s}_D
	\approx 0.96$.
Using the structural decomposition, $\lambda \cdot \sigma \approx
	0.05$--$0.07$, so $\bar{s}_K^{\text{sem}} - \bar{s}_D \approx
	0.89$--$0.91$.
The structural term is amplified exponentially, explaining the
$\sim$3$\times$ attention-per-token ratio.
For LLaMA-3-8B: ratio $= 0.762$, $\bar{s}_K - \bar{s}_D \approx
	0.65$ (lower $\lambda \approx 0.03$--$0.06$).

\subsubsection{Structural Capture Potential: Formal Definition}

For triples, $\sigma$ is elevated by relation keyword density, repeated
slot patterns, and low verb diversity.
$\hat{\sigma}(\text{KG}) \approx 0.70$;
$\hat{\sigma}(\text{neutral}) \approx 0.25$.

\subsection{Source--Task Alignment: Full Decomposition}
\label{app:mi_decomposition}

This appendix provides the full mutual-information decomposition
summarised in Section~\ref{sec:theory_info}.

The influence of $K$ on the model's output can be decomposed via the
chain rule of mutual information:
\begin{equation}
	I(y;\, K \mid q, D) =
	\underbrace{I(y;\, K_{\text{rel}} \mid q, D)}_{\text{useful signal}}
	+ \underbrace{I(y;\, K_{\text{irr}} \mid q, D, K_{\text{rel}})}_{\text{distraction}}.
	\label{eq:mi_decomp}
\end{equation}
Partitioning $K = K_{\text{rel}} \cup K_{\text{irr}}$ into
task-relevant and irrelevant components, the first term captures
signal and the second captures distraction.

For task-misaligned sources (e.g., ConceptNet for Wikipedia-based
questions), $K_{\text{rel}} \approx \varnothing$ and distraction
dominates:
\begin{equation}
	I(y;\, K \mid q, D) \approx
	I(y;\, K_{\text{irr}} \mid q, D) \geq 0.
	\label{eq:mi_misaligned}
\end{equation}
The structural attention tax amplifies this: even the distraction
component receives elevated attention due to
$\lambda \cdot \sigma(K)$, meaning misaligned triples are not merely
uninformative but \emph{actively costly} because they impose a
format-driven attention tax on top of semantic distraction.

\section{Extended Discussion}
\label{app:discussion}

\subsection{Extended Limitations}
\label{app:extended_limitations}

\textbf{Causality:} First-token log-prob conflates multiple factors.
\textbf{Retrieval:} Cosine similarity optimises surface match.
\textbf{Coverage:} Two 7B/8B models, 4-bit, greedy, three tasks.
\textbf{Evaluation:} Exact match may penalise correct but lexically
different answers.
\textbf{Attention:} Last-layer, correlational, no causal intervention.
\textbf{$\sigma(K)$:} Discrete pattern-matching; better understood as
relative measure.
\textbf{Strategy~S5:} Assumes stable training dynamics under the
attention penalty.
\textbf{Generality:} The structural tax concept is demonstrated only
for KG triple format in this study; extension to SQL, JSON, code, and
other structured formats is a natural but unverified hypothesis.

\subsection{Qualitative Error Analysis}
\label{app:qualitative}

\textbf{CSQA KG-Broken:} ``Where are traveling clothes often
kept?''---\texttt{bedroom closet IsA closet} names a competing
location, overriding \emph{luggage}.
High $\sigma$ ensures attention capture; $\bar{s}_K^{\text{sem}}$
aligned with wrong answer makes the structural tax actively harmful.

\textbf{HotpotQA KG-Fixed:} ``Are both Lygodium and Maxillaria
orchids?''---\texttt{lygodium IsA fern genus} supplies the missing
fact.
$\bar{s}_K^{\text{sem}}$ aligned with correct answer; the structural
tax is benign because captured attention carries useful signal.

\textbf{TriviaQA KG-Broken:} ``Rapidly boiling a liquid to make it
thicker\ldots''---triples shift from \emph{Reduction} to
\emph{Concentrating}.
The structural tax amplifies the effect of semantically misleading
content.

\section{Reproducibility: Experiment-to-Code Reference}
\label{app:reproducibility}

This section maps each experimental result in the paper to the script in the
\texttt{code/} directory that produces it.
All scripts assume the \texttt{icl\_kg} conda environment and read
model/dataset paths from \texttt{code/config.py}.
Run scripts from the \texttt{code/} directory unless otherwise noted.
The quickest end-to-end entry point is \texttt{run\_pipeline.py}
(or \texttt{run\_c5\_c6.bat} for the C5/C6 mitigation experiments only).

\subsection*{Data Preparation}
\begin{itemize}
  \item \texttt{01\_prepare\_data.py} — builds \texttt{data/csqa\_200.json},
        \texttt{hotpotqa\_200.json}, \texttt{triviaqa\_200.json}.
  \item \texttt{02\_retrieve\_kg.py} — queries ConceptNet and caches triples
        in \texttt{kg\_cache/}.
  \item \texttt{prepare\_triviaqa\_local.py} /
        \texttt{prepare\_triviaqa\_from\_rc.py} — alternative loaders for
        local or RC-format TriviaQA copies.
\end{itemize}

\subsection*{Main Accuracy Results}
\begin{itemize}
  \item \textbf{Table~1} (C1–C4 accuracy, $n{=}200$, both models, all tasks):
        \texttt{03\_run\_experiment.py} $\to$
        \texttt{results/\{model\}/\{task\}\_results.json};
        formatted by \texttt{04\_analyze\_results.py} $\to$
        \texttt{results/\{model\}/main\_table.tex}.
  \item \textbf{$k$-shot ablation} (Table~\ref{tab:kshot}):
        \texttt{03\_run\_experiment.py -{}-ablation} $\to$
        \texttt{\{task\}\_k\_ablation.json};
        table by \texttt{04\_analyze\_results.py} $\to$
        \texttt{k\_ablation\_table.tex}.
  \item \textbf{Answer log-probabilities} (Table~\ref{tab:logprob}):
        \texttt{04\_analyze\_results.py} $\to$ \texttt{logprob\_table.tex}.
  \item \textbf{Per-sample error flow} (Table~\ref{tab:errorflow}):
        \texttt{04\_analyze\_results.py} $\to$
        \texttt{analysis\_summary.json}.
  \item \textbf{KG-gain heatmap} (Figure~\ref{fig:heatmap}): inline
        \textsc{TikZ} in \texttt{main.tex}; numeric values from
        \texttt{04\_analyze\_results.py}.
\end{itemize}

\subsection*{Statistical Tests}
\begin{itemize}
  \item \textbf{McNemar at $n{=}200$} (main text) and
        \textbf{$n{\approx}500$} (Table~\ref{tab:mcnemar500}):
        \texttt{07\_statistical\_corrections.py} $\to$
        \texttt{results/statistical\_corrections.\{json,tex\}};
        expanded data from \texttt{11\_expand\_samples.py} and
        \texttt{15\_expand\_csqa.py}.
  \item \textbf{McNemar at $n{=}1{,}000$} ($\dagger$ cell in
        Figure~\ref{fig:heatmap}): \texttt{20\_expand\_n1000.py},
        then re-run \texttt{07\_statistical\_corrections.py}.
  \item \textbf{SARP threshold stability} (Section~\ref{app:sarp}):
        \texttt{12\_sarp\_stability.py} $\to$
        \texttt{results/sarp\_stability.\{json,tex\}}.
  \item \textbf{Attention Wilcoxon tests}:
        \texttt{13\_attention\_stats.py} $\to$
        \texttt{results/attention\_stats.\{json,tex\}}.
\end{itemize}

\subsection*{Attention Analysis}
\begin{itemize}
  \item \textbf{Attention tables} (Tables~\ref{tab:attention},
        \ref{tab:attention_llama}):
        \texttt{09\_attention\_analysis.py} $\to$
        \texttt{results/attention\_analysis\_\{model\}.json}.
  \item \textbf{Multi-layer attention}
        (Tables~\ref{tab:multilayer_mistral},~\ref{tab:multilayer_llama};
        Figures~\ref{fig:multilayer_mistral},~\ref{fig:multilayer_llama}):
        \texttt{14\_multilayer\_attention.py} $\to$
        \texttt{results/multilayer\_attention\_\{model\}.json} and
        \texttt{figures/fig\_multilayer\_attention\_\{model\}.pdf}.
  \item \textbf{Structural attention capture figure}
        (framework diagram, Figure~1):
        \texttt{25\_structural\_attention\_analysis.py} $\to$
        \texttt{figures/fig\_structural\_attention\_capture.pdf}.
\end{itemize}

\subsection*{Ablation and Robustness}
\begin{itemize}
  \item \textbf{Noise validation} (Table~\ref{tab:noise}):
        \texttt{08\_noise\_validation.py} $\to$
        \texttt{results/noise\_validation.\{json,tex\}}.
  \item \textbf{Difficulty stratification}
        (Tables~\ref{tab:diff_mistral}–\ref{tab:diff_indep}):
        \texttt{10\_oracle\_analysis.py} $\to$
        \texttt{results/\{model\}/difficulty\_\{task\}.tex}.
  \item \textbf{FP16 vs.\ INT4} (Table~\ref{tab:fp16}):
        \texttt{22\_fp16\_experiment.py} $\to$
        \texttt{results/fp16\_comparison.\{json,tex\}}.
  \item \textbf{Alias-aware evaluation} (Section~\ref{app:alias}):
        \texttt{23\_alias\_aware\_eval.py} $\to$
        \texttt{results/alias\_aware\_comparison.\{json,tex\}}.
\end{itemize}

\subsection*{Mitigation Experiments}
\begin{itemize}
  \item \textbf{C5 neutral-text decomposition} (Table~\ref{tab:c5_decomp}):
        \texttt{16\_neutral\_text\_c5.py} $\to$
        \texttt{results/\{model\}/\{task\}\_c5\_results.json};
        table by \texttt{18\_c5\_analysis.py} $\to$
        \texttt{results/c5\_analysis.\{json,tex\}}.
  \item \textbf{C5b verbalized triples} (Table~\ref{tab:c5b}):
        \texttt{24\_c5b\_verbalized.py} $\to$
        \texttt{results/\{model\}/\{task\}\_c5b\_results.json};
        table also produced by \texttt{18\_c5\_analysis.py}.
  \item \textbf{C5b attention comparison} (Table~\ref{tab:c5b_attn}):
        \texttt{27\_c5b\_attention\_analysis.py} $\to$
        \texttt{results/attention\_analysis\_c5b\_\{model\}.json}.
  \item \textbf{S1 structural dispersal} (Table~\ref{tab:s1}):
        \texttt{28\_s1\_dispersal\_experiment.py} $\to$
        \texttt{results/\{model\}/\{task\}\_s1\_dispersal\_results.json}
        and \texttt{results/attention\_analysis\_s1\_\{model\}.json}.
  \item \textbf{C6 multi-feature gating} (Table~\ref{tab:oracle}, C6 column):
        \texttt{17\_multifeature\_gating.py} $\to$
        \texttt{results/multifeature\_gating\_c6.\{json,tex\}}.
  \item \textbf{Oracle upper bound} (Table~\ref{tab:oracle}, Oracle column):
        \texttt{10\_oracle\_analysis.py} $\to$
        \texttt{results/oracle\_analysis.\{json,tex\}}.
  \item \textbf{BM25 retrieval baseline} (Section~\ref{sec:bm25}):
        \texttt{21\_bm25\_retrieval.py} $\to$
        \texttt{results/bm25\_comparison.\{json,tex\}}.
\end{itemize}

\subsection*{Figures and Utilities}
\begin{itemize}
  \item \textbf{Main figures} (accuracy bars, $k$-ablation lines, error-flow
        alluvial, difficulty plots, log-prob scatter):
        \texttt{05\_visualize.py} $\to$ \texttt{results/figures/}.
  \item \textbf{Qualitative error examples} (Section~\ref{app:qualitative}):
        \texttt{06\_qualitative\_analysis.py} $\to$
        \texttt{results/\{model\}/qualitative\_examples\_\{task\}.json}.
  \item \texttt{26\_consolidate\_results.py} — aggregates all result files
        into \texttt{results/RESULTS\_CONSOLIDATED.md}.
  \item \texttt{generate\_final\_results.py} — human-readable summary across
        all experiments.
\end{itemize}

\end{document}